%% file: main.tex
\documentclass{article}

\PassOptionsToPackage{round}{natbib}
\usepackage[final]{neurips_2021}

\input{packages}

\input{notations}

\newif\ifsup
\supfalse\suptrue		% Keep Appendix
\title{Stochastic Multi-Armed Bandits with \\ Control Variates}

\author{
	Arun Verma \\
 	Department of Computer Science \\ 
 	National University of Singapore \\%, Republic of Singapore \\
 	\texttt{arun@comp.nus.edu.sg}
  	\And
  	Manjesh K. Hanawal \\
  	Department of IEOR \\
  	IIT Bombay, Mumbai,	India \\
  	\texttt{mhanawal@iitb.ac.in}
 }

\begin{document}

	\maketitle

	\begin{abstract}
        This paper studies a new variant of the stochastic multi-armed bandits problem where auxiliary information about the arm rewards is available in the form of control variates. In many applications like queuing and wireless networks, the arm rewards are functions of some exogenous variables. The mean values of these variables are known a priori from historical data and can be used as control variates.  Leveraging the theory of control variates, we obtain mean estimates with smaller variance and tighter confidence bounds. We develop an upper confidence bound based algorithm named UCB-CV and characterize the regret bounds in terms of the correlation between rewards and control variates when they follow a multivariate normal distribution. We also extend UCB-CV to other distributions using resampling methods like Jackknifing and Splitting. Experiments on synthetic problem instances validate performance guarantees of the proposed algorithms.	
    \end{abstract}

	\section{Introduction}
	\label{sec:introduction}
	\input{introduction}

	\section{Control Variate for Variance Reduction}
	\label{sec:control_variate}
	\input{control_variate}

	\section{Problem Setting}		
	\label{sec:problem_setting}
	\input{problem_setting}

	\section{Arms with Normally Distributed Rewards and Control Variates}
	\label{sec:normal_dist}
	\input{normal_dist}

	\section{General Distribution}
	\label{sec:no_dist}
	\input{no_dist}

	\section{Experiments}
	\label{sec:experiments}
	\input{experiment}

	\section{Conclusion}
	\label{sec:conclusion}
	\input{conclusion}

    \section*{Acknowledgements}
    Manjesh K. Hanawal is supported by INSPIRE faculty fellowship from DST and Early Career Research Award (ECR/2018/002953) from SERB, Govt. of India. This work was done when Arun Verma was at IIT Bombay.

	\bibliographystyle{plainnat} % Other options: unsrtnat, unsrt, abbrvnat, plainnat, rusnat
	\bibliography{ref}

	\ifsup
		\newpage
		\onecolumn
		\centerline{\Large \bf Appendix} 
		\vspace{-2mm}
		\hrulefill 
		
		\appendix
		\label{sec:appendix}
		\input{appendix}
	\fi
	
\end{document}

%% file: packages.tex
\usepackage[markup=underlined]{changes}
\usepackage[utf8]{inputenc}   % allow utf-8 input
\usepackage[T1]{fontenc}    	% use 8-bit T1 fonts
\usepackage{url}              		 % simple URL typesetting
\usepackage{booktabs}       	% professional-quality tables
\usepackage{amsfonts}       	% blackboard math symbols
\usepackage{nicefrac}       	  % compact symbols for 1/2, etc.
\usepackage{microtype}      	% microtypography

\usepackage{graphicx}
\usepackage{color}
\usepackage{rotating}
\usepackage{tabularx}
\usepackage{pdflscape}
\usepackage{amsmath,amsthm,amssymb}
\usepackage{algorithm, algorithmic}
\usepackage{bbm,dsfont}
\usepackage{subfig}
\usepackage{caption}
\usepackage{wrapfig}
\usepackage{cancel}
\usepackage{enumerate, cases}
\usepackage{thmtools,thm-restate}
\usepackage{mathtools}
\usepackage[round]{natbib}

\usepackage[none]{hyphenat}
\usepackage{multirow}
\usepackage{makecell}
\usepackage{setspace}
\usepackage{pifont}
\usepackage{array}
\usepackage{enumitem}

\allowdisplaybreaks

\usepackage{hyperref}		 % hyperlinks
\hypersetup{
	colorlinks	= true,		 % Colours links instead of ugly boxes
	urlcolor     = blue,	 % Colour for external hyperlinks
	linkcolor	 = purple, 	 % Colour of internal links
	citecolor    = violet    % Colour of citations
}
\usepackage[capitalize]{cleveref}

%% file: notations.tex
\newcommand{\Regret}{\kR}

\newcommand{\EE}[1]{\bE\left[#1\right]}

\newcommand{\Prob}[1]{\bP\left\{#1\right\}}

\newcommand{\R}{\bR}
\newcommand{\N}{\bN}

\newcommand{\one}[1]{\mathds{1}_{\left\{#1\right\}}}

\renewcommand{\phi}{\varphi}
\renewcommand{\epsilon}{\varepsilon}

\newcommand{\al}[1]{ \begin{align} #1  \end{align}}
\newcommand{\eq}[1]{ \begin{equation} #1  \end{equation}}
\newcommand{\als}[1]{ \begin{align*} #1  \end{align*}}
\newcommand{\eqs}[1]{ \begin{equation*} #1  \end{equation*}}

\newcommand{\el}{\end{flushleft}}
\newcommand{\bl}{\begin{flushleft}}
\newcommand{\floor}[1]{ \left\lfloor #1 \right\rfloor }
\newcommand{\ceil}[1]{ \left\lceil #1 \right\rceil }

\newcommand{\argmax}{\arg\!\max}

\newcommand{\bE}{\mathbb{E}}

\newcommand{\bN}{\mathbb{N}}

\newcommand{\bP}{\mathbb{P}}

\newcommand{\bR}{\mathbb{R}}

\newcommand{\cN}{\mathcal{N}}

\newcommand{\cP}{\mathcal{P}}

\newcommand{\kR}{\mathfrak{R}}

\theoremstyle{plain}

\newtheorem{lem}{Lemma}

\newtheorem{rem}{Remark}

\newtheorem{fact}{Fact}

%% file: introduction.tex
%!TEX root =  main.tex

Stochastic Multi-Armed Bandits (MAB) setting has been extensively used to study decision making under uncertainty \citep{BIOMETRIKA33_thompson1933likelihood, NOW12_bubeck2012regret,lattimore_szepesvari_2020}. In the classical setting, it is assumed that the arm rewards are independent of each other. After playing an arm, the learner observers independent and identically distributed reward samples. The exploration versus exploitation trade-off is a fundamental problem in the bandits setting, and a learner can accumulate more rewards if it is balanced well. A better balance can be achieved if the learner can estimate arm's mean rewards with tighter confidence bounds \citep{ML02_auer2002finite,PMH10_auer2010ucb,COLT11_garivier2011kl} and smaller variance \citep{TCS09_audibert2009exploration,AAAI18_mukherjee2018efficient}. Any available side or auxiliary information can aid in building tighter confidence bounds. In this paper, we study the improvement in the performance that can be achieved when side information is available in the form of control variates.

Any observable input that is correlated with the random variable of interest can be a Control Variate (CV). When the mean of the CV is known, it becomes useful to reduce the variance of the mean estimator of the random variable. CV method is a popular variance reduction technique to improve the estimates' precision without altering the simulation runs or an increase in the number of samples \citep{OR82_lavenberg1982statistical}. Many stochastic simulations involve inputs generated using known distributions. These inputs can be potential CVs as output could be correlated with them. It has motivated the development of rich theory to analyze the quality of the estimators resulting from CVs \citep{MS82_lavenberg1981perspective, OR82_lavenberg1982statistical, OR90_nelson1990control}. We adopt these techniques to study the stochastic multi-armed bandits problem when CVs are available for the arms.

In many real-life applications, exogenous variables influence arms' rewards and act as CVs. For example, consider a job aggregating platform that assigns jobs to one of the servers/workers in each slot. The jobs are of variable size, and the time to serve them depends on their size and other unknown extraneous factors. The platform observes the size of each job and knows their mean value from historical data. The random job sizes correlate with service times and can be used as a CV to estimate each server's mean service time. Another example is a wireless system where a transmitter can transmit packets on a channel in each slot. Successful transmission of the packets (quality) depends on random quantities like fading, interference, and channel noise on the selected channel. The transmitter can observe the fading value using pilot signals in each round and know their mean value from past observations. These fading values on a channel can be used as a CV to estimate the channel's quality.

The performance of a MAB algorithm depends on how good is the confidence interval of its mean reward estimators \citep{ML02_auer2002finite, PMH10_auer2010ucb, COLT11_garivier2011kl}. The tightness of these intervals depends on the variance of the estimators \citep{TCS09_audibert2009exploration, AAAI18_mukherjee2018efficient}. Naturally,  estimators with smaller variance for the same number of samples result in better performance. Thus CVs can be leveraged to improve the performance of bandits algorithms with smaller confidence intervals.  The reduction depends on how strongly the reward samples and associated CVs are correlated.

Linear methods are widely used for variance reduction using CVs where the new samples are generated taking the weighted sum of reward samples and centered CV samples. Then these new samples are used for the estimation of the mean rewards. The choice of the weight that results in maximum variance reduction depends on the variance of CV and its correlation with reward samples. In practice, both of these quantities are unknown and need to be estimated from the observed samples. However, using the estimated weight makes the new samples highly correlated and makes the analysis of confidence intervals challenging.

For multivariate normal distributions, tight confidence intervals can be constructed that hold with high probability for the mean reward estimators resulting from the linear control variates. Using these results, we develop a MAB algorithm that exploits CVs to improve the regret performance. For general distributions, we discuss resampling methods that result in unbiased estimators with better confidence intervals. Specifically, our contributions are as follows:
\begin{itemize}
	\item For the case where rewards and CVs are normally distributed, we develop an algorithm named Upper Confidence Bounds with Control Variates (\ref{alg:UCB-CV}) that uses the estimators based on the linear control variates. 
	
	\item In \cref{sec:normal_dist}, we show that the regret of \ref{alg:UCB-CV} is smaller by a factor $(1-\rho^2)$ compared to the existing algorithms when the rewards and CV have a multivariate normal distribution, where $\rho$ is the correlation coefficient of the reward and control variates.
	
	\item In \cref{sec:no_dist}, we discuss how to adapt UCB-CV for the general distributions using estimators and associated confidence intervals based on jackknifing, splitting, and batching methods.
	
	\item We validate the performance of \ref{alg:UCB-CV} on synthetically generated data in \cref{sec:experiments}.
\end{itemize}

\subsection{Related Work}
Our work uses side or auxiliary information available in the form of CVs to improve the variance of estimates. In the following, we discuss works that incorporate variance estimates and/or side information to improve the performance of multi-armed bandits algorithms.

{\bf Incorporating variance estimates:} Many stochastic multi-armed bandits algorithms like UCB1 \citep{ML02_auer2002finite}, DMED \citep{COLT10_honda2010asymptotically}, KL-UCB \citep{COLT11_garivier2011kl}, Thompson Sampling \citep{NIPS11_chapelle2011empirical,COLT12_agrawal2012analysis,ALT12_kaufmann2012thompson,AISTATS13_agrawal2013further}, assume that the rewards are bounded on some interval $[0,b]$ and define index of each arm based on its estimates of the mean rewards and ignore the variance estimates. The regret of these algorithms have optimal logarithmic bound in the number of rounds $(T),$ and grows quadratic in $b$, i.e., $\mathcal{O}((b^2/\Delta)\log T)$, where $\Delta$ is the smallest sub-optimality gap. UCB1-NORMAL  by \citep{ML02_auer2002finite} uses the variance estimates in the index of the arms and when rewards are Gaussian with variance $\sigma^2$, to achieve regret is of order $\mathcal{O}((\sigma^2/\Delta)\log T)$. This bound is even better than that of UCB1 when the variance is smaller than $b^2$ even though it has considered unbounded support. Further, the authors experimentally demonstrate that the UCB-Tuned algorithm \citep{ML02_auer2002finite}, which uses variance estimates in arm indices, significantly improves regret performance over UCB1 when rewards are bounded. Extending the idea of using `variance aware' indices, UCBV  \citep{TCS09_audibert2009exploration} algorithm achieves regret of order $\mathcal{O}((\sigma^2/\Delta+2b)\log T)$ thus confirming that making algorithms variance aware is beneficial. EUCBV \citep{AAAI18_mukherjee2018efficient} improves the performance of UCBV using an arm elimination strategy like used in UCB-Improved \citep{PMH10_auer2010ucb}.

{\bf Incorporating side information:} In the applications like advertising and web search, whether or not users like a displayed item could depend on their profile. Users' profile is often available to the platform, which could then be used as side information. Contextual bandits \citep{COLT08_dani2008stochastic, MOR10_rusmevichientong2010linearly, WWW10_li2010contextual,NIPS10_filippi2010parametric, NIPS11_abbasi2011improved, AISTATS11_chu2011contextual,ICML17_li2017provably, NIPS17_jun2017scalable, ICML16_zhang2016online} and MAB with covariates \citep{AS13_perchet2013multi} exploit such side information to learn the optimal arm. They assume a correlation structure (linearity, GLM, etc.) between mean reward and context where mean reward varies with side-information (context).

In contrast to contextual bandits, mean rewards do not vary in our unstructured setup with side information available in the form of CVs. We do not use CVs directly to decide which arm to select in each round but use them only to get better estimates of the mean rewards by exploiting their correlation with rewards. Our motivating examples from jobs scheduling and communication networks capture these unstructured bandits where CVs provide information about the rewards but do not alter their mean values.

Our confidence bounds for mean reward estimators are based on the variance of estimators rather than on the variance of the reward sample, as was the case in UCBV \citep{TCS09_UCBV_Audibert} and EUCBV \citep{AAAI18_mukherjee2018efficient}. The samples that we use to estimate mean rewards are not independent and identically distributed. Hence, we can not use the standard Bernstein inequality to get confidence bounds on mean and variance estimates.

CVs are used extensively for variance reduction \citep{JORS85_james1985variance, Willy14_botev2014variance,ICML16_chen2016scalable,ACL17_kreutzer2017bandit,ICML19_vlassis2019design} in the Monte-Carlo simulation of complex systems \citep{MS82_lavenberg1981perspective,OR82_lavenberg1982statistical,EJOR89_nelson1989batch,OR90_nelson1990control}. To the best of our knowledge, our work is the first to analyze the performance of stochastic bandits algorithms with  control variates.

%% file: control_variate.tex
%!TEX root =  main.tex

Let $\mu$ be the unknown quantity that needs to be estimated, and $X$ be an unbiased estimator of $\mu$, i.e., $\EE{X} = \mu$. A random variable $W$ with known expectation $(\omega)$ is a CV if it is correlated with $X$. Linear control methods use errors in estimates of known random variables to reduce error in the estimation of an unknown random variable. For any choice of a coefficient $\beta$, define a new  estimator as
$
    \bar{X} = X + \beta(\omega-W).
$

It is straightforward to verify that its variance is given by
\als{
    \text{Var}(\bar{X}) 
    &= \text{Var}(X) + \beta^2\text{Var}(W) - 2\beta\text{Cov}(X,W).  \label{equ:varNewEst}
}
and it is minimized by setting $\beta$ to 
\eqs{
    \label{equ:optimalBeta}
    \beta^\star = \frac{\text{Cov}(X,W)}{\text{Var}(W)}.
}
The minimum value of the variance is given by $\text{Var}(\bar{X}) = (1 - \rho^2)\text{Var}(X)$, where $\rho$ is the correlation coefficient of $X$ and $W$. Larger the correlation, the greater the variance reduction achieved by the CV. In practice, the values of $\text{Cov}(X,W)$ and $\text{Var}(W)$ are unknown and need to be estimated to compute the best approximation for $\beta^\star$.

%% file: problem_setting.tex
%!TEX root =  main.tex

We consider a stochastic multi-armed bandits problem with $K$ arms. The set of arms is represented by $[K] \doteq \{1, 2, \ldots, K\}$. In round $t$, the environment generates a vector $(\{X_{t,i}\}_{i \in [K]}, \{W_{t,i}\}_{i \in [K]})$. The random variable  $X_{t, i}$ denotes the reward of arm $i$ in round $t$, and form an Independent and Identically Distributed (IID) sequence drawn from an unknown but fixed distribution with mean $\mu_i$ and variance $\sigma_i^2$. The random variable $W_{t, i}$ is the Control Variate (CV) associated with the reward of arm $i$ in round $t$. These random variables are drawn from an unknown but fixed distribution with mean $\omega_i$ and variance $\sigma_{w, i}^2$ and form an IID sequence. The learner knows the values $\{\omega_i\}_{i \in [K]}$ but not the $\{\sigma_{w, i}^2\}_{i \in [K]}$.  The correlation coefficient between rewards and associated control variates of an arm $i$ is denoted by $\rho_i$.

In our setup, the learner observes the reward and associated CVs from the selected arm. We refer to this new variant of the multi-armed bandits problem as Multi-Armed Bandits with Control Variates (MAB-CV). The parameter vectors $\boldsymbol{\mu^\sigma} = \{\mu_i, \sigma_i^2\}_{i \in [K]}$, $\boldsymbol{\omega^\sigma} = \{\omega_i, \sigma_{w,i}^2\}_{i \in [K]}$, and $\boldsymbol{\rho} = \{\rho_i\}_{i \in [K]}$ identify an instance of MAB-CV problem, which we denote henceforth using $P = (\boldsymbol{\mu^\sigma}, \boldsymbol{\omega^\sigma}, \boldsymbol{\rho})$. The collections of all MAB-CV problems are denoted by $\cP$. For a problem instance $P \in \cP$ with mean reward vector $\boldsymbol{\mu}$, we denote the optimal arm as $i^\star = \argmax\limits_{i \in [K]} \mu_i.$

\begin{algorithm}[!ht]
    \renewcommand{\thealgorithm}{MAB-CV problem}
	\floatname{algorithm}{}
	\caption{with instance $(\boldsymbol{\mu^\sigma}, \boldsymbol{\omega^\sigma}, \boldsymbol{\rho})$}
	\label{alg:BwCV}
	For round $t$: 
	\begin{enumerate}
		\item \textbf{Environment} generates a vector $\boldsymbol{X_t} = (X_{t,1},\ldots,$ $X_{t,K}) \in \R^K$ and $\boldsymbol{W_t} = (W_{t,1},\ldots, W_{t,K}) \in \R^K$, where $\EE{X_{t,i}}=\mu_i$, Var$(X_{t,i}) = \sigma_i^2$, $\EE{W_{t,i}}=\omega_i$, Var$(W_{t,i}) = \sigma_{w,i}^2$, and the sequence $(X_{t,i})_{t\geq 1}$ and $(W_{t,i})_{t\geq 1}$ are IID for all $i\in [K]$.
		\item \textbf{Learner} selects an arm $I_t \in [K]$ based on past observation of rewards and CVs samples from the arms till round $t-1$.
		\item \textbf{Feedback and Regret:} The learner observes a random reward $X_{t,I_t}$ and associated CV  $W_{t,I_t}$, and then incurs penalty $(\mu_{i^\star} - \mu_{I_t})$.
	\end{enumerate}
\end{algorithm}

Let $I_t$ denote the arm selected by the learner in round $t$ based on the observation of past reward and control variate samples. The interaction between the environment and a learner is given in \ref{alg:BwCV}. We aim to learn policies that accumulate maximum reward and measure its performance by comparing its cumulative reward with that of an Oracle that plays the optimal arm in each round. Specifically, we measure the performance in terms of regret defined as follows:
\begin{equation}
	\label{eqn:regret}
	\Regret_T = T\mu_{i^\star} - \EE{\sum_{t=1}^{T} X_{t,I_t}}.
\end{equation}
Note that maximizing the mean cumulative reward of a policy is the same as minimizing the policy's regret. A good policy should have sub-linear expected regret, i.e., ${\Regret_T}/T \rightarrow 0$ as $T \rightarrow \infty$. 

Our goal is to learn a policy that is sub-linear with small regret. To this effect, we use CVs to estimate mean rewards with sharper confidence bounds so that the learner can have a better exploration-exploitation trade-off and start playing the optimal arm more frequently early.

%% file: normal_dist.tex
%!TEX root =  main.tex

We first focus on the case where the rewards and associated CVs of arms have a multivariate normal distribution. We discuss the general distributions in the next section. To bring out the main ideas of the algorithm, we first consider the case where each arm is associated with only one CV. Motivated by the linear CV technique discussed in the previous section, we consider a new sample for an arm $i$ in round $t$ as follows:
\eq{
    \label{equ:observedSample}
    \bar{X}_{t,i} = X_{t,i} + \beta^*_i(\omega_i - W_{t,i}),
}
where $X_{t,i}$ is the $t^{\text{th}}$ reward, $W_{t,i}$ is the $t^{\text{th}}$ associated control variate with an arm $i$, $\omega_i=\EE{W_{t,i}}$, and $\beta_i^*=\text{Cov}(X_{t,i},W_{t,i})/\text{Var}(W_{t,i})$. Using $s$ such samples, the mean reward for an arm $i$ is defined as
$\hat\mu_{s,i}^c = (\sum_{r=1}^s \bar{X}_{r,i})/s.$
Let $\hat{\mu}_{s,i} = \frac{1}{s} \sum_{r=1}^s X_{r,i}$ and $\hat{\omega}_{s,i} = \frac{1}{s} \sum_{r=1}^s W_{r,i}$ denote the sample mean of reward and CVs of an arm $i$ from $s$ samples. Then $\hat\mu_{s,i}^c$ can be written as follows:
\eq{
    \label{equ:estAvgReward}
	\hat\mu_{s,i}^c = \hat\mu_{s,i} + \beta^*_{s,i}(\omega_i - \hat\omega_{s,i}).
}
Since the value of $\text{Cov}(X_i,W_i)$ and $\text{Var}(W_i)$ are unknown, the optimal coefficient $\beta_{s,i}^*$ need to be estimated. After having $s$ samples, it is  estimated as below and used in \cref{equ:estAvgReward}. 
\eq{
    \label{equ:estBeta}
    \hat\beta^*_{s,i}= \frac{\sum_{r=1}^{s} (X_{r,i} - \hat\mu_{s,i})(W_{r,i} - \omega_i)}{\sum_{r=1}^{s}(W_{r,i}-\omega_i)^2}.
}

When the variance of CV is known, it can be directly used to estimate $\beta^*_i$ by replacing estimated variance by actual variance of CV in the denominator of \cref{equ:estBeta}, i.e., replacing $\sum_{r=1}^{s}(W_{r,i}-\omega_i)^2/ (s-1)$ by $\sigma_{w,i}^2$. After observing reward and associated control variates for arm $i$, the value of $\hat\beta_i^\star$ is re-estimated. The value of $\hat\beta_i^\star$ depends on all observed rewards and associated control variates from arm $i$. Since all $\bar{X}_{., i}$ uses same $\hat\beta_i^\star$, this leads to correlation between the $\bar{X}_{., i}$ observations.

Let $\text{Var}(\bar{X}_i) = \sigma_{c,i}^2$ denote the variance of the linear CV samples and $\nu_{s,i}=\text{Var}(\hat{\mu}^c_{s,i})$ denote the variance the estimator using these samples. Since the mean reward estimator $(\hat{\mu}^c_{s,i})$ is computed using the correlated samples, we cannot use Bernstein inequality to get confidence bounds on mean reward estimates. However, when the rewards and CVs follow a multivariate normal distribution, the following results tell how to obtain an unbiased variance estimator and confidence interval for $\hat{\mu}_{s, i}^c$.
\begin{restatable}{lem}{varEst}
	\label{lem:varEst}
	Let the reward and control variate of each arm have a multivariate normal distribution.  After observing $s$ samples of reward and control variate from arm $i$, define
	$
	\hat\nu_{s,i} = \frac{Z_{s,i} \hat\sigma_{c,i}^2(s)}{s},
	$
	where $$Z_{s,i} = \left(1 - \frac{\left(\sum_{r=1}^{s} (W_{r,i} - \omega_i)\right)^2}{s\sum_{r=1}^{s} (W_{r,i} - \omega_i)^2}\right)^{-1} ~\text{ and }~~ \hat\sigma_{c,i}^2(s)= \frac{1}{s-2}\sum_{r=1}^{s} (\bar{X}_{r,i} - \hat\mu_{s,i}^c)^2,$$ then $\hat{\nu}_{s,i}$ is an unbiased variance estimator of $\hat{\mu}^c_{s,i}$, i.e., $\mathbb{E}[\hat{\nu}_{s,i}]=\text{Var}(\hat{\mu}^c_{s,i})$.
\end{restatable}

The confidence interval for reward estimator is given by our next result.
\begin{restatable}{lem}{confProb}
    \label{lem:confProb}
    Let the conditions in \cref{lem:varEst} hold and $s$ be the number of reward and associated control variate samples from arm $i$ in round $t$. Then
    \eqs{
        \Prob{|\hat\mu_{s,i}^c - \mu_i| \ge V_{t,s}^{(\alpha)}\sqrt{{\hat\nu_{s,i}}}} \le 2/t^\alpha,
    }
    where $V_{t,s}^{(\alpha)}$ denote $100(1-1/t^\alpha)^{\text{th}}$ percentile value of the $t-$distribution with $s-2$ degrees of freedom and   $\hat\nu_{s,i}$ is an unbiased estimator for  variance  of $\hat{\mu}^c_{s,i}$.
\end{restatable}
We prove \cref{lem:varEst} and \cref{lem:confProb} using regression theory and control theory results \citep{OR90_nelson1990control}. The detailed proofs are given in \cref{asec:oneCV}. The proof uses the fact that the arm's rewards and CVs have a multivariate normal distribution. Therefore, we can use $t$-distribution for designing confidence intervals of mean reward estimators. Analogous to the Hoeffding inequality, this result shows that the probability of the estimate obtained by samples $\{\bar{X}_{t, i}\}$ deviating from the true mean of arm $i$ decays fast. Further, the deviation factor $V_{t,s}^{(\alpha)}\sqrt{{\hat\nu_{s,i}}}$ depends on the variance of the estimator $(\hat{\mu}^c_{s,i})$, which guarantees sharper confidence intervals. As we will see later in Lemma \ref{lem:estProp}, these confidence terms are smaller by a factor $(1-\rho_i^2)$ compared to the case when no CVs are used. Equipped with these results, we next develop a Upper Confidence Bound (UCB) based algorithm for the MAB-CV problem.

\subsection{Algorithm: \ref{alg:UCB-CV}}
Let $N_i(t)$ be the number of times the arm $i$ is selected by a learner until round $t$ and $\hat\nu_{N_i(t),i}$ be the unbiased sample variance of mean reward estimator $(\hat\mu_{N_i(t),i}^c)$ for arm $i$. Motivated from Lemma \ref{lem:confProb}, we define optimistic upper bound for mean reward estimate of arm $i$ as follows: 
\begin{equation}
	\label{equ:UCB}
	\text{UCB}_{t,i} = \hat\mu_{N_i(t),i}^c + V_{t,N_i(t)}^{(\alpha)}\sqrt{{\hat\nu_{N_i(t),i}}}.
\end{equation}

Using above values as UCB indices of arms, we develop an algorithm named \ref{alg:UCB-CV} for the MAB-CV problem. The algorithm works as follows: It takes the number of arms $K$, a constant $Q=3$ (number of  CV per arm $+~ 2$), and $\alpha>1$ (trades-off between exploration and exploitation) as input. Each arm is played $Q$ times to ensure the sample variance for observations $(\hat\sigma_{c,i}^2(s))$ can be computed (see Lemma \ref{lem:varEst}).

\begin{algorithm}[!ht] 
	\renewcommand{\thealgorithm}{UCB-CV}
	\floatname{algorithm}{}
	\caption{UCB based Algorithm for MAB-CV problem}
	\label{alg:UCB-CV}
	\begin{algorithmic}[1]
		\STATE \textbf{Input:} $K, ~Q, ~\alpha>1$
		\STATE Play each arm $i \in [K]$ $Q$ times
		\FOR{$t=QK+1, QK + 2, \ldots, $}
			\STATE $\forall i \in [K]:$ compute UCB$_{t-1,i}$ as given in \cref{equ:UCB}
			\STATE Select $I_t = \argmax\limits_{i \in [K]}$ UCB$_{t-1,i}$
			\STATE Play arm $I_t$ and observe $X_{t,I_t}$ and associated control variates $W_{t,I_t}$. Increment the value of $N_{I_t}(t)$ by one and re-estimate $\hat{\beta}^*_{N_{I_t}(t),I_t}$, $\hat\mu_{N_{I_t}(t),{I_t}}^c$ and $\hat\nu_{t,N_{I_t}(t)}$
		\ENDFOR
	\end{algorithmic}
\end{algorithm}

In round $t$, \ref{alg:UCB-CV} computes the upper confidence bound of each arm's mean reward estimate using \cref{equ:UCB} and selects the arm having highest upper confidence bound value. We denote the selected arm by $I_t$. After playing arm $I_t$, the reward $X_{t, I_t}$ and associated control variate $W_{t, I_t}$ are observed. After that, the value of $N_{I_t}(t)$ is updated and $\hat{\beta}^*_{N_{I_t}(t),I_t}$, $\hat\mu_{N_{I_t}(t),{I_t}}^c$ and $\hat\nu_{t,N_{I_t}(t)}$ are re-estimated. The same process is repeated for the subsequent rounds.

\subsection{Estimator with Multiple Control Variates}
In some applications, it could be possible that each arm is associated with multiple CVs. We denote the number of CVs with each arm as $q$. Let $W_{t,i,j}$ be the $j^{\text{th}}$ control variate of arm $i$ that is observed in round $t$, . Then the unbiased mean reward estimator for arm $i$ with associated CVs is given by
\eqs{
	\hat\mu_{s,i,q}^c = \hat\mu_{s,i} + \boldsymbol{\hat\beta}_{i}^{*\top} (\boldsymbol{\omega}_{i} - \boldsymbol{\hat\omega}_{s,i}),
}
where $\boldsymbol{\hat\beta}^*_i = \left(\hat\beta_{i,1}, \ldots, \hat\beta_{i,q}\right)^\top$, $\boldsymbol{\omega}_{i} = \left(\omega_{i,1}, \ldots, \omega_{i,q}\right)^\top$, and $\boldsymbol{\hat\omega}_{s,i}= \left(\hat\omega_{s,i,1}, \ldots, \hat\omega_{s,i,q}\right)^\top$.

Let $s$ be the number of rewards and associated CVs  samples for arm $i$,  $\boldsymbol{W}_i$ be the $s\times q$ matrix whose $r^{\text{th}}$ row is $\left(W_{r,i,1}, W_{r,i,2}, \ldots, W_{r,i,q} \right)$, and $\boldsymbol{X}_i=(X_{1,i}, \ldots, X_{s,i})^\top$. By extending the arguments used in \cref{equ:estBeta} to $q$ control variates, the estimated coefficient vector is given by 
\eqs{
    \boldsymbol{\hat\beta}^*_i = (\boldsymbol{W}_i^\top\boldsymbol{W}_i - s\boldsymbol{\hat\omega}_{s,i}\boldsymbol{\hat\omega}_{s,i}^\top )^{-1} (\boldsymbol{W}_i^\top \boldsymbol{X}_i - s{\boldsymbol{\hat\omega}_i}~\hat\mu_{s,i}).
}

We can generalize \cref{lem:varEst} and \cref{lem:confProb} for the MAB-CV problems with $q$ control variates and then use \ref{alg:UCB-CV} with $Q=q+2$ and appropriate optimistic upper bound for multiple control variate case. More details can be found in \cref{asec:multiCV}.

\subsection{Analysis of \ref{alg:UCB-CV}}
In our next result, we describe the property of our estimator that uses control variates. This result is derived from the standard results of control variates theory \citep{OR90_nelson1990control}.
\begin{restatable}{lem}{estProp}
    \label{lem:estProp}
    Let reward and control variates have a multivariate normal distribution and $q$ be the number of control variates. Then after having $s$ observations of rewards and associated control variates from arm $i$,
    \als
    {
        &\EE{\hat\mu_{s,i,q}^c} = \mu_i, \text{ and}\\
        & \text{Var}(\hat\mu_{s,i,q}^c) = \frac{s-2}{s-q-2}(1-{\rho^2_i})\text{Var}(\hat\mu_{s,i}),
    }
    where ${\rho^2_i} = \sigma_{X_i\boldsymbol{W}_i}\Sigma_{\boldsymbol{W}_i\boldsymbol{W}_i}^{-1}\sigma_{X_i\boldsymbol{W}_i}^\top/\sigma_{i}^2$ is the square of the multiple correlation coefficient, $\sigma_{i}^2 = \text{Var}(X_i)$, and $\sigma_{X_i\boldsymbol{W}_i} = (\text{Cov}(X_i,W_{i,1}), \ldots, \text{Cov}(X_i,W_{i,q}))$.
\end{restatable}
 
The following is our main result which gives the regret upper bound of \ref{alg:UCB-CV}. The detailed analysis is given in \cref{asec:analysis}.
\begin{restatable}{thm}{regretBound}
	\label{thm:regretBound}	
	Let the conditions in \cref{lem:estProp} hold, $\alpha=2$, $\Delta_i = \mu_{i^\star} - \mu_i$ be the sub-optimality gap for arm $i \in [K]$, and $N_i(T)$ be the number of times sub-optimal arm $i$ selected in $T$ rounds. Let $C_{T,i,q} = \EE{\frac{N_i(T)- 2}{N_i(T)-q-2}\left(\frac{V_{T,N_i(T),q}^{(2)}}{V_{T,T,q}^{(2)}}\right)^2}$ for all $i$. Then the regret of \ref{alg:UCB-CV} in $T$ rounds is upper bounded by
	\eqs{
		\Regret_T \le \sum_{i \ne i^\star} \hspace{-0.5mm} \left( \frac{4(V_{T,T,q}^{(2)})^2 C_{T,i,q}(1-\rho_i^2)\sigma_i^2}{\Delta_i} + \frac{\Delta_i\pi^2}{3} + \Delta_i\right).
	}
\end{restatable}

Note that $\alpha$ is a hyper-parameter to trade-off between exploration and exploitation. \ref{alg:UCB-CV} works as long as $\alpha>1$ and the effect of $\alpha$ on the regret bound is through the term $\sum_{i=1}^\infty \frac{1}{t^\alpha}$. This term has a nice closed form value of $\pi^2/6$ for $\alpha=2$. Hence, we state the results with $\alpha=2$ and use it in the experiments as well. 

\begin{rem}
	Unfortunately, we cannot directly compare our bound with that of UCB based stochastic MAB algorithms like UCB1-NORMAL and UCBV as $V_{T, T, q}^{(2)}$ do not have closed-form expression. $V_{T, T, q}^{(2)}$ denotes $100(1-1/T^2)^\text{th}$ percentile of $t$-distribution with $T-q-1$ degrees of freedom.
\end{rem}

\begin{rem}
As the degrees of freedom decreases in $q$, $V_{T, T, q}^{(2)}$ increases in $q$ (see \cref{fig:VTTq_value_log} for small value of $T$ and \cref{fig:VTTq_value_log_largeT} for large value of $T$). However, this does not imply that the overall regret increases in $q$ as $\rho_i^2$ also increases in $q$. Hence dependency of regret on $q$ is delicate. It is also observed in  \cite{EJOR89_nelson1989batch} that having more CVs does not mean better variance reduction. One can empirically verify that $(V_{T, T, q}^{(2)})^2$ is upper bounded by ${3.726 \log T}$. This upper bound holds for all $q$ as long as $T \ge q + \max(32, 0.7q)$.
\end{rem}
\begin{figure}[H]
	\centering
	\subfloat[{\tiny $(V_{T,T,q}^{(2)})^2$ vs $3.726\log(T)$ for small $T$.}]{
		\label{fig:VTTq_value_log}	
		\includegraphics[scale=0.28]{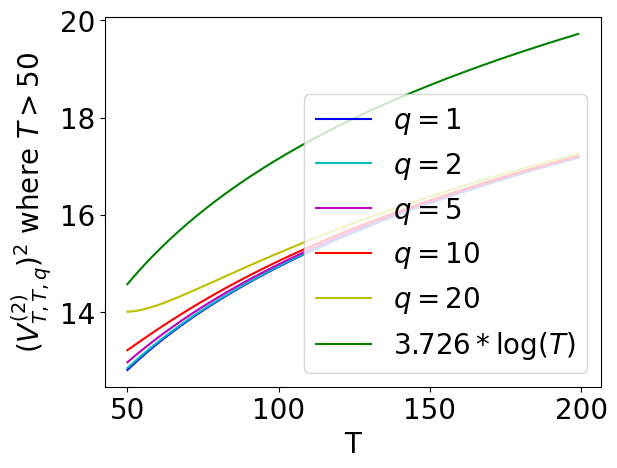}}
	\subfloat[{\tiny $(V_{T,T,q}^{(2)})^2$ vs $3.726\log(T)$ for large $T$.}]{
		\label{fig:VTTq_value_log_largeT}	
		\includegraphics[scale=0.28]{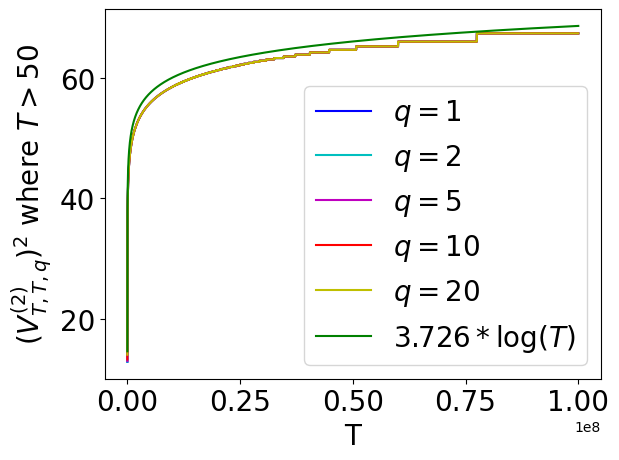}}
	\subfloat[{\tiny Variation in $\frac{V_{T,N_i(T),q}^{(2)}}{V_{T,T,q}^{(2)}}$ with $N_i(T)$.}]{\label{fig:Ratio}
		\includegraphics[scale=0.28]{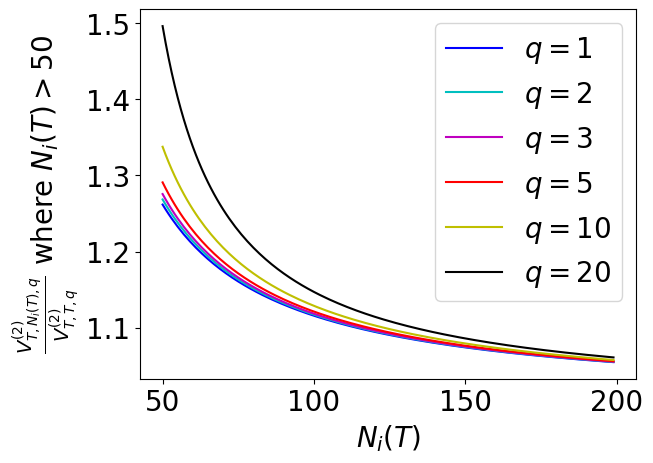}}
	\vspace{-1mm}
	\caption{Properties of $(V_{T,T,q}^{(2)})^2$ and ${V_{T,N_i(T),q}^{(2)}}/{V_{T,T,q}^{(2)}}$.}
\end{figure}

\vspace{-3mm}
\begin{rem}
	It is hard to bound $C_{T, i, q}$ as $V_{T, \cdot, q}^{(2)}$ does not have a explicit form. However, $C_{T, i, q}$ tends to 1 as $T \rightarrow \infty$. Indeed, from lower bound argument, we know that  $ N_i(T) \rightarrow \infty$ as $T \rightarrow \infty$ and hence $\frac{N_i(T) - 2}{N_i(T)-q-2} \searrow 1$ and $\frac{V_{T,N_i(T),q}^{(2)}}{V_{T,T,q}^{(2)}} \searrow 1$ (confidence interval shrink with more samples) as shown in \cref{fig:Ratio}. However, under some assumptions, we can obtain a bound empirically. Specifically, if $q \le 20$ and $N_i(T) > 50$, then $\frac{N_i(T) - 2}{N_i(T)-q-2}\leq 1.72$ and from $t$-distribution tables, we get $\frac{V_{T,N_i(T),q}^{(2)}}{V_{T,T,q}^{(2)}}\leq 1.5$ (see \cref{fig:Ratio}), hence the bound $C_{T, i, q}\leq 4$. We can guarantee $N_i(T)>50$ by playing each arm $50$ times at the beginning. Therefore, we can obtain one possible explicit bound in Lemma (6) by replacing $C_{T,i,q}$ with $4$. However, for a general case, an explicit bound in \cref{lem:expectedPulls} is hard.
\end{rem}

%% file: no_dist.tex
%!TEX root =  main.tex

We now consider the general case with no distributional assumptions on reward and associated CVs of arms. The samples $\bar{X}_{t, i,q}$ resulting from the linear combination of rewards and CVs are dependent but need not be normally distributed. Hence $\hat\mu_{s,i,q}^c$ need not remain an unbiased estimator. Therefore, we cannot use the $t$-distribution properties to obtain confidence intervals. However, we can use resampling methods such as jackknifing, splitting, and batching to reduce the bias of the estimators and develop confidence intervals that hold approximately or asymptotically. These confidence intervals can then be used to defined indices for the arms. Below we briefly discuss these methods.

\subsection{Jackknifing}
\label{ssec:jackknifing}
In statistics, jackknifing is a well-known re-sampling technique for variance and bias estimation \citep{Book82_efron1982jackknife, TAS83_efron1983leisurely}. We start with a classical presentation of jackknifing for CVs \citep{OR82_lavenberg1982statistical}. Let $\hat\mu_{s,i,q}^{c,-j}$ be the estimate of the mean reward that is computed without sample $\bar{X}_j$. The  $Y_{j,i,q}^{\text{J}} = s\hat\mu_{s,i,q}^{c} - (s-1)\hat\mu_{s,i,q}^{c,-j}$, which is sometimes called the $j^{\text{th}}$ pseudo-value. Using pseudo-values mean reward estimation  given as $\hat\mu_{s,i,q}^{c, \text{J}} = \frac{1}{s} \sum_{j=1}^{s} Y_{j,i,q}^{\text{J}}$ and its sample variance is $\hat\nu_{s,i,q}^{\text{J}} = (s(s-1))^{-1}$ $\sum_{j=1}^s (Y_{j,i,q}^{\text{J}} - \hat\mu_{s,i,q}^{c, \text{J}})^2$. $\hat\mu_{s,i,q}^{c, \text{J}}$ is shown to be unbiased \citep{EJOR89_nelson1989batch}[Thm. 7] when reward variance is bounded. Further, $\hat\mu_{s,i,q}^{c, \text{J}} \pm t_{\alpha/2}(n-1)\hat\nu_{s,i,q}^{\text{J}}$ is a asymptotically valid confidence interval \citep{NRL91_avramidis1991simulation} and holds approximately for finite number of samples. 

\subsection{Splitting}
\label{ssec:splitting}
Splitting is a technique to split the correlated observations into two or more groups, compute an estimate of $\boldsymbol{\beta}_i^*$ from each group, then exchange the estimates among the groups. \citet{OR90_nelson1990control} considers an extreme form of splitting, where he splits $s$ observations into $s$ groups. The $j^{\text{th}}$ observation is given by
\eqs{
    Y_{j,i,q}^{\text{S}} = X_{j,i} + \boldsymbol{\hat\beta}_i^{*-j}(\boldsymbol{\omega} - \boldsymbol{W}_{j,i}), ~~~ j \in [s],
}
where $\boldsymbol{\hat\beta}_i^{*-j}$ is estimated without $j^{\text{th}}$ observation and $\boldsymbol{W}_{j,i} = \left(W_{j,i,1}, \ldots, W_{j,i,q} \right)$ is the vector of CVs with $q$ elements associated with reward $X_{j,i}$. The point estimator for splitting method  is $\hat\mu_{s,i,q}^{c, \text{S}} = \frac{1}{s} \sum_{j=1}^{s} Y_{j,i,q}^{\text{S}}$ and its sample variance is $\hat\nu_{s,i,q}^{\text{S}} = (s(s-1))^{-1}$ $\sum_{j=1}^s (Y_{j,i,q}^{\text{S}} - \hat\mu_{s,i,q}^{c, \text{S}})^2$. Then $\hat\mu_{s,i,q}^{c, \text{S}} \pm t_{\alpha/2}(n-1)\hat\nu_{s,i,q}^{\text{S}}$ gives an approximate confidence interval \citep{OR90_nelson1990control}. 

Below we define UCB index for these methods.  Let $\hat\nu_{N_i(t),i,q}^{\text{G}}$ be the sample variance of mean reward estimator $(\hat\mu_{s,i,q}^{c,\text{G}})$ for arm $i$. We define the optimistic upper bound for mean reward for  $G \in \{J, S\}$ as follows: 
\eq{
    \label{equ:genDistUCB}
	\text{UCB}_{t,i,q}^{\text{G}} = \hat\mu_{N_i(t),i,q}^{c, \text{G}} + V_{t,N_i(t),q}^{\text{G},\alpha}\sqrt{\hat\nu_{N_i(t),i,q}^{\text{G}}}.
}
where $V_{t, s, q}^{\text{G}, \alpha}$ is the $100(1-1/t^\alpha)^{\text{th}}$ percentile value of the $t-$distribution with $s-1$ degrees of freedom.
Since the optimistic upper bounds defined in \cref{equ:genDistUCB} are valid only asymptotically \citep{OR90_nelson1990control, NRL91_avramidis1991simulation}, it cannot be used for any finite time regret guarantee.
However the above UCB indices can be used in \ref{alg:UCB-CV} to get an heuristic algorithm for the general case. We experimentally validate its performance in the next section.

%% file: experiment.tex
%!TEX root =  main.tex

We empirically evaluate the performance of \ref{alg:UCB-CV}  by comparing it with \ref{alg:UCB-CV} with UCB1 \citep{ML02_auer2002finite}, UCB-V \citep{TCS09_audibert2009exploration}, Thompson Sampling \citep{AISTATS13_agrawal2013further}, and EUCBV \citep{AAAI18_mukherjee2018efficient} on different synthetically generated problem instances. For all the instance we use we use $K=10$, $q=1$, and $\alpha=2$. All the experiments are repeated $100$ times and cumulative regret with a $95\%$ confidence interval (the vertical line on each curve shows the confidence interval) are shown.  Details of each instance are as follows:

{\it Instance 1:} The reward and associated CV  of this instance have a multivariate normal distribution. The reward of each arm has two components. We treated one of the components as CV. In round $t$, the reward of arm $i$ is given as follows:
\eqs{
    X_{t,i} = V_{t,i} + W_{t,i},
}
where $V_{t,i} \sim \cN(\mu_{v,i}, \sigma_{v,i}^2)$ and $W_{t,i} \sim \cN(\mu_{w,i}, \sigma_{w,i}^2)$. Therefore, $X_{t,i} \sim \cN(\mu_{v,i}+\mu_{w,i}, \sigma_{v,i}^2 + \sigma_{w,i}^2)$. We treat $W_{i}$ as CV of $X_i$. It can be easily shown that the correlation coefficient of $X_i$ and $W_i$ is $\rho_i = \sqrt{\sigma_{w,i}^2/(\sigma_{v,i}^2+\sigma_{w,i}^2)}$. For each $i$, we set $\mu_{v,i} = 0.6 - (i-1)*0.05$, $\mu_{w,i} = 0.3$, for arm $i \in [K]$. The value of $\sigma_{v,i}^2 = 0.1$ and $\sigma_{w,i}^2 = 0.1$ for all arms. Note that the problem instance has the same CV for all arms, but all the CVs observations for each arm are maintained separately. 

{\it Instance 2:} It is the same as the Instance $1$ except each arm has a different CV associated with its reward. The mean value of the CV associated with arm $i$ is set as $\mu_{w,i} = 0.8 - (i-1)*0.05$.

{\it Instance 3:}  It is the same as the Instance $2$ except the samples of reward and associated control variate are generated from Gamma distribution, where the value of scale is set to $1$.

{\it Instance 4:}  It is the same as the Instance $2$ except the samples of reward and associated control variate are generated from Log-normal distribution with the values of $\sigma_{v,i}^2 = 1$ and $\sigma_{w,i}^2 = 1$ for all arms.

\paragraph{Comparing regret of \ref{alg:UCB-CV} with existing algorithms} 
The regret of different algorithms for Instances $1$ and $2$ are shown in Figures \cref{fig:same} and \cref{fig:diff}, respectively. As see \ref{alg:UCB-CV} outperforms all the algorithms. We observe that Thomson Sampling has a large regret for normally distributed reward and CVs. Hence, we have not added regret of Thomson Sampling. Note that \ref{alg:UCB-CV} does not require a bounded support assumption, which is needed in all other algorithms. The regret of the EUCBV algorithm is closest to \ref{alg:UCB-CV}, but EUCBV can stick with the sub-optimal arm as it uses an arm elimination-based strategy, which has a small probability of eliminating the optimal arm.

\begin{figure}[!ht]
	\centering
	\subfloat[Instance 1]{\label{fig:same}
		\includegraphics[scale=0.3]{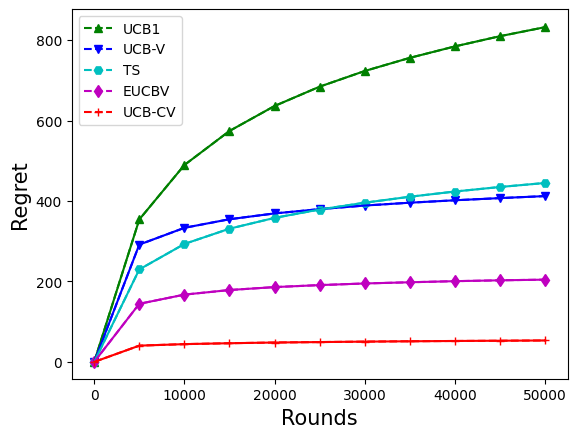}}
	\subfloat[Instance 2]{\label{fig:diff}
		\includegraphics[scale=0.3]{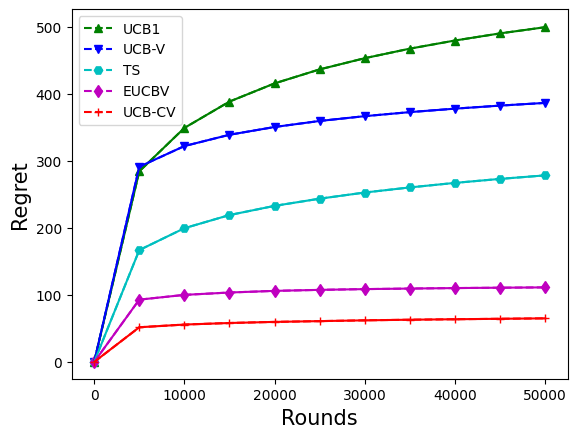}}
	\subfloat[Varying Correlation Coefficient.]{\label{fig:regretCorr}
		\includegraphics[scale=0.3]{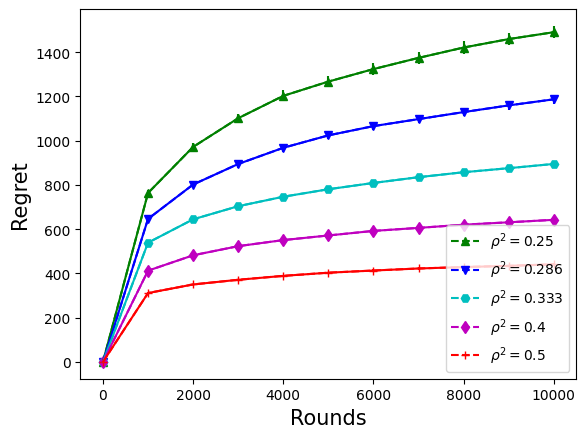}}
	
	\caption{Regret comparison of different multi-armed bandits algorithms with \ref{alg:UCB-CV}.}
	\label{fig:compRegret}
\end{figure}

\paragraph{Regret of \ref{alg:UCB-CV} vs Correlation Coefficient}
Theorem \ref{thm:regretBound} shows that \ref{alg:UCB-CV} have better regret bounds when correlation between arm rewards and CVs is higher. To validate this we derived problem instances having different correlation coefficients as follows:

{\it Instance 5:} Similar to the Instance $1$, the reward and associated CVs have a multivariate normal distribution with rewards expressed as sum of two components.  We set  $\mu_{v,i} = 6.0 - (i-1)*0.5$ and $\mu_{w,i} = 4.0$ for arm $i \in [K]$. The value of $\sigma_{v,i}^2 = 1.0$ and $\sigma_{w,i}^2 = 1.0$ for all arms. The problem instance has common CV for all arms but all the observations are maintained for each arm separately. As the correlation coefficient of $X_i$ and $W_i$ is $\rho_i = \sqrt{\sigma_{w,i}^2/(\sigma_{v,i}^2+\sigma_{w,i}^2)}$, we  varied $\sigma_{v,i}^2$ over the values $\{1.0, 1.5., 2.0, 2.5, 3.0\}$ to obtain problem instances with different correlation coefficient. Increasing $\sigma_{v,i}^2$ reduces the correlation coefficient of $X_i$ and $W_i$. We have plotted the regret of \ref{alg:UCB-CV} for different problem instances. As expected, we observe that the regret decreases as the correlation coefficient of reward and its associated control variates increases as shown in \cref{fig:regretCorr}.

\paragraph{Performance of Jackknifing, Splitting, and Batching Methods} We compare the regret of \ref{alg:UCB-CV} with Jackknifing, Splitting and Batching (see \cref{asec:batching} for more details) methods. The regret of the batching method is worse for all the instances. Whereas, the jackknifing and splitting are performing well for heavy-tailed distributions (Instance $3$) and \ref{alg:UCB-CV} performs better for normal distribution (Instance $2$) and Log-normal distribution (Instance $4$) as shown in \cref{fig:genDist}. 
\vspace{-5mm}
\begin{figure}[!ht]
	\centering
	\subfloat[Instance 2]{\label{fig:normal}
		\includegraphics[scale=0.3]{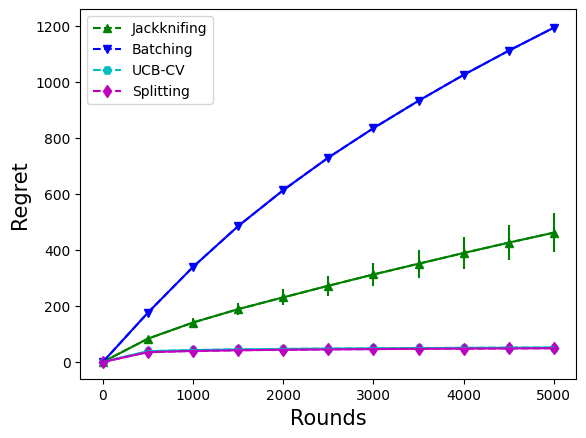}}
	\subfloat[Instance 3]{\label{fig:gamma}
		\includegraphics[scale=0.3]{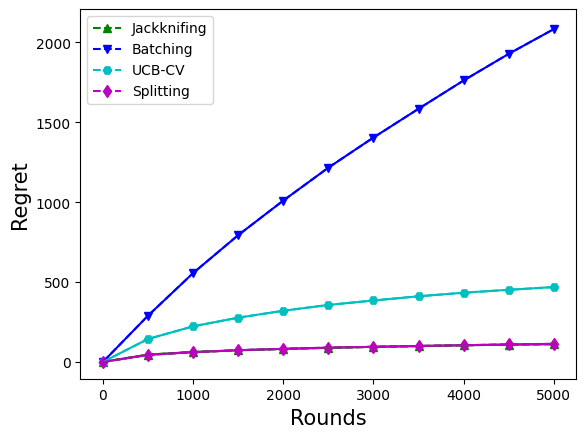}}
	\subfloat[Instance 4.]{\label{fig:log-normal}
		\includegraphics[scale=0.3]{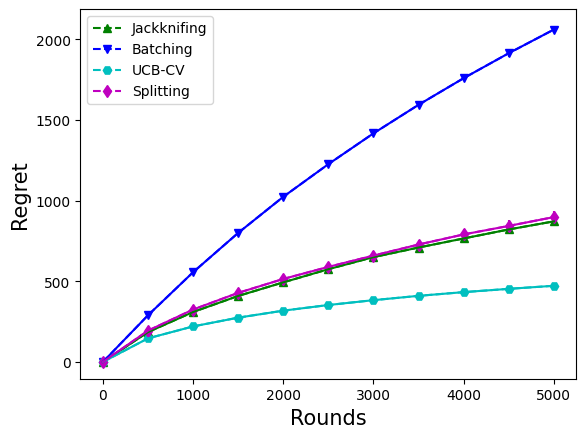}}
		
	\caption{Comparing regret of \ref{alg:UCB-CV} with Jackknifing, Splitting, and Batching method based.}
	\label{fig:genDist}
\end{figure}

\paragraph{Regret of \ref{alg:UCB-CV} vs Estimated Mean of CV}

Our work has demonstrated the applicability of CVs to bandit theory. To establish the gains analytically, we have to assume that mean of the CVs is known. However, it is not unreasonable as the mean of CVs can be constructed such that its mean value is known (see Eq. (7) and Eq. (8) of \cite{ACL17_kreutzer2017bandit}). If the mean of CV is unknown, we estimate it from the samples, but using the estimated/approximated means can deteriorate the performance of the proposed algorithm.s Though the theory of control variate is well developed for known mean, and only empirical studies are available to the case when the CV means are known approximately \citep{IEE01_schmeiser2001biased, IIE12_pasupathy2012control}.

To know the effect of approximation error $(\epsilon)$ in the mean estimation of CVs, we run an experiment with Instance $1$, Instance $2$, and Instance $5$. We assume that the approximated mean of CVs is given by $\omega_i + \epsilon$. We observe that the regret of UCB-CV increases with an increase in approximation error. UCB-CV can even start performing poorly than the existing state-of-the-art algorithm for significant large approximation errors as shown in \cref{fig:RegretVsError}. Since the maximum and minimum reward gap can be more than one for the Instance $5$, we must multiply confidence intervals of UCB1, UCB-V, and EUCBV with the appropriate factor (used $6$ for the experiment) for a fair comparison. We also observe that Thompson Sampling has almost linear regret for the Instance $5$. We believe that it is due to the high variance of arms' rewards that leads to overlapping of rewards distributions.
\vspace{-5mm}
\begin{figure}[H]
	\centering
	\subfloat[Instance 1]{\label{fig:Instance1}
		\includegraphics[scale=0.3]{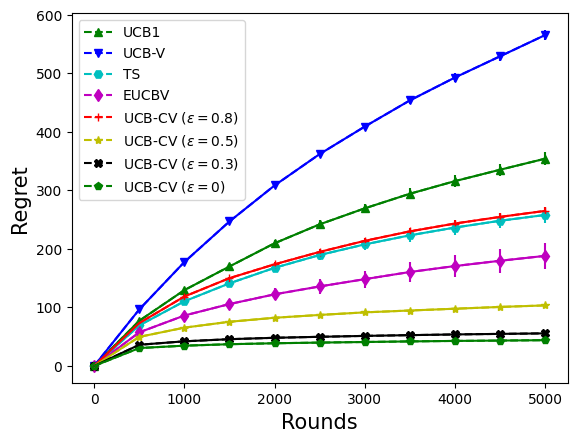}}
	\subfloat[Instance 2]{\label{fig:Instance2}
		\includegraphics[scale=0.3]{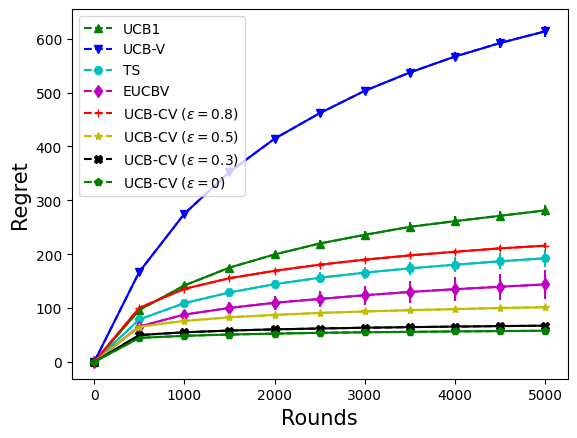}}
	\subfloat[Instance 5.]{\label{fig:Instance5}
		\includegraphics[scale=0.3]{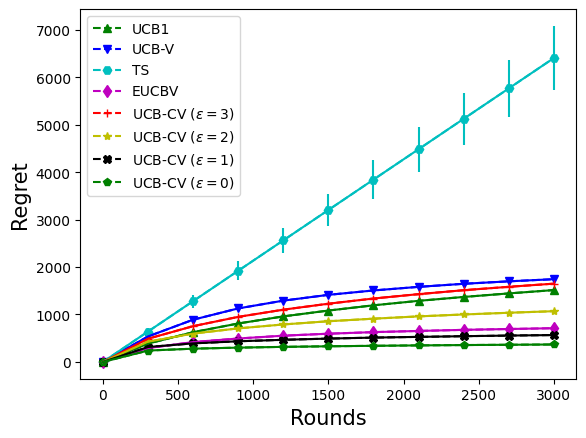}}
	
	\caption{Regret comparison of \ref{alg:UCB-CV} with varying approximation error in the mean estimation of control variates and different multi-armed bandits algorithms.}
	\label{fig:RegretVsError}
\end{figure}
\vspace{-3mm}
Our experimental results demonstrate that as long as the error in the estimation of the mean of CVs is within a limit, there will be an advantage using control variates. The impact of such approximations in bandits needs to be analyzed extensively and demands independent work.

%% file: conclusion.tex
%!TEX root =  main.tex

In this work, we studied stochastic multi-armed bandits problem with side information available in the form of Control Variate (CV).  Leveraging the linear control variates' variance reduction properties, we developed a variant of the UCB algorithm named UCB-CV. When reward and CVs have a multivariate normal distribution, we showed that the regret of UCB-CV, which is of the order $\mathcal{O}(\sigma^2(1-\rho^2)/\Delta\log T)$ where $\sigma^2$ is the maximum variance of arm rewards and $\rho$ is the minimum correlation between the arm rewards and CVs. As expected, our the bounds showed that when the correlation is strong, one can get significant improvement in regret performance. For the case where the reward and control variates follow a general distribution, we discussed Jackknifing and splitting resampling that will give estimators with smaller bias and sharper confidence bounds.

The variance reduction techniques based on CVs assume that the exact mean of the CVs are known.  
%We will adapt these results to the bandits setting in our future study.
In practice, only some rough estimates of the mean of CVs rather than the exact values may be available.  It would be interesting to study how does it affect the performance of the bandit algorithms. Another interesting direction to study how CVs are helpful if one has to rank the arms instead of just finding the top-ranked arm.

%% file: appendix.tex
%!TEX root =  main.tex

\section{Control Variates and Regression Theory}
Consider the following regression problem with $n$ samples and $p$ features:
\eqs{
	Y_i = \boldsymbol{X}_i^\top\boldsymbol{\theta} + \epsilon_i, ~~~ i \in \{1,2, \ldots, n\}
}
where $Y_i \in \R$ is the $i^{\text{th}}$ target variable, $ \boldsymbol{X}_i \in \R^p$ is the $i^{\text{th}}$ feature vector, $\boldsymbol{\theta} \in \R^p$ is the unknown regression parameters, and $\epsilon_i$ is a normally distributed noise with mean $0$ and constant variance $\sigma^2$. The values of noise $\epsilon_i$ form a IID sequence and are independent of $\boldsymbol{X}_i $. Let

$$
\boldsymbol{Y} = \begin{pmatrix}
	Y_1 \\
	\vdots \\
	Y_n
\end{pmatrix},
\qquad
\boldsymbol{X}^\top = \begin{pmatrix}
	X_{11}~~ \ldots~~ X_{1p}\\
	\vdots  ~~~~~~ \cdots ~~~~~~\vdots\\
	X_{n1}~~\ldots ~~X_{np}
\end{pmatrix}, \text{ and}
\quad
\boldsymbol{\epsilon} = \begin{pmatrix}
	\epsilon_1 \\
	\vdots \\
	\epsilon_n
\end{pmatrix}.
$$

The least square estimator is given by $\boldsymbol{\hat\theta} = (\boldsymbol{X}\boldsymbol{X}^\top)^{-1}\boldsymbol{X}\boldsymbol{Y}$. Next we give the finite sample properties of $\boldsymbol{\hat\theta}$ that are useful to prove the regret upper bounds.

\begin{fact}
	\label{fact:estmProp}
	The finite sample properties of the least square estimator of  $\boldsymbol{\theta}$:
	\als{
		&1.~\EE{\boldsymbol{\hat\theta} | \boldsymbol{X}} = \boldsymbol{\theta},&~\text{(unbiased estimator)} \\
		&2.~\text{Var}(\boldsymbol{\hat\theta} | \boldsymbol{X}) = \sigma^2(\boldsymbol{X}\boldsymbol{X}^\top)^{-1}, \text{ and} ~& \text{(expression for the variance)} \\
		&3.~\text{Var}(\boldsymbol{\hat\theta}_i | \boldsymbol{X}) = \sigma^2(\boldsymbol{X}\boldsymbol{X}^\top)_{ii}^{-1}, ~& \text{(element-wise variance)} 
	}
	where $(\boldsymbol{X}\boldsymbol{X}^\top)_{ii}^{-1}$ is the $ii-$element of the matrix  $(\boldsymbol{X}\boldsymbol{X}^\top)^{-1}$.
\end{fact}

The first two properties are derived form Proposition 1.1 of \citep{Book_hayashi2000econometrics}, whereas the third property is taken from \citep{ESBS05_van2005estimation}. 
Next we give the finite sample properties of  estimator of variance $\sigma^2$.
\begin{fact}{\citep[Proposition 1.2]{Book_hayashi2000econometrics}}
	\label{fact:varEst}
	Let $\hat\sigma^2 = \frac{1}{n-p} \sum_{i=1}^n(Y_i -\boldsymbol{X}_i^\top\boldsymbol{\hat\theta})^2$ be estimator of $\sigma^2$ and $n>p$ (so that $\hat\sigma^2$ is well defined).  Then $\EE{\hat\sigma^2 | \boldsymbol{X}} = \sigma^2$ which implies that $\hat\sigma^2 $ is a unbiased estimator of $\hat\sigma^2$. 
\end{fact}

Now recall the unknown parameter estimation problem with $q$ control variates and $t$ observations:
\eqs{
	\bar{X}_{s} = X_{s} +  \boldsymbol{\hat\beta}^{*\top} (\boldsymbol{\omega} - \boldsymbol{W}_{s}), ~~~ s \in \{1,2,\ldots, t\}.
}
where the subscript for arm index $i$ and number of control variates $q$ are dropped for simplicity. Under the assumption of multivariate normal distribution, we can write $\bar{X}_{s} $ as follows:
\eqs{
	\bar{X}_{s} = \mu +  \boldsymbol{\beta}^{*\top} (\boldsymbol{\omega}- \boldsymbol{W}_{s}) + \epsilon_s,
}
where $\epsilon_1, \ldots, \epsilon_t$ are IID normally distributed random variables with mean $0$ and variance $\sigma^2$. Let
$$
\boldsymbol{\bar{X}} = \begin{pmatrix}
	\bar{X}_1 \\
	\vdots \\
	\bar{X}_t
\end{pmatrix},
\qquad
\boldsymbol{Y} = \begin{pmatrix}
	1 &\omega_{1} - W_{11} ~ \ldots~ \omega_{q} - W_{1q}\\
	\vdots &\vdots  ~~~~~~ \cdots ~~~~~~\vdots\\
	1 &\omega_{1} - W_{t1} ~\ldots ~ \omega_{q} -  W_{tq}
\end{pmatrix}, 
\quad
\boldsymbol{\gamma} = \begin{pmatrix}
	\mu \\
	\beta^*
\end{pmatrix},\text{ and}
\quad
\boldsymbol{\epsilon} = \begin{pmatrix}
	\epsilon_1 \\
	\vdots \\
	\epsilon_t
\end{pmatrix}.
$$

Now let the least square estimator of $\mu$ be $\hat\mu$ and $\beta^\star$ be $\hat\beta^\star$.  Then using \cref{fact:estmProp},
\eqs{
	\text{Var}(\hat\mu) = \sigma^2 (\boldsymbol{Y}^\top\boldsymbol{Y})_{11}^{-1}
}
where $(\boldsymbol{Y}^\top\boldsymbol{Y})_{11}^{-1}$ is the upper left most element of matrix $(\boldsymbol{Y}^\top\boldsymbol{Y})^{-1}$ \citep{OR82_schmeiser1982batch}. 
Recall definitions of $\boldsymbol{S}_{W_iW_i}$, $\hat\mu_{t,i}^c$, and $\hat\omega_i$ (and ignore the arm index for now), Then after $t$ observations, the estimator of Var$(\hat\mu)$ is given by
\eq{
	\label{equ:varMulti}
	\hat\nu_t =  \frac{Z_{t} \hat\sigma^2(t)}{t}
} 
where $Z_{t} = \left(1 + \frac{(\boldsymbol{\hat\omega}_{t} - \boldsymbol{\omega})^\top\boldsymbol{S}_{WW}^{-1} (\boldsymbol{\hat\omega}_{t} - \boldsymbol{\omega})}{1-1/t}\right)$ and $ \hat\sigma^2(t)= \frac{1}{t-q-1}\sum_{s=1}^{t} (\bar{X}_s - \hat\mu_{t}^c)^2$ \citep{OR90_nelson1990control}. Using \cref{fact:varEst}, $\hat\sigma^2(t)$ is an unbiased estimator of $\sigma^2(t)$.
Next, we give the fundamental results  of control variates estimators.
\begin{fact}{\citep[Theorem 1]{OR90_nelson1990control}}
	\label{fact:normalEstProp}
	Let $Z_s = (X_s, W_{s1}, \ldots, W_{sq})^\top$. Suppose that $\{Z_1, \ldots, Z_t\}$ are IID $(q+1)-$variate normally distributed vectors with mean vector $(\mu, \boldsymbol{\omega})^\top$ and variance $\Sigma_{ZZ}$. Then
	\als
	{
		&\EE{\hat\mu_{t}^c} = \mu, \\
		& \text{Var}(\hat\mu_{t}^c) = \frac{t-2}{t-q-2}(1-{\rho^2})\text{Var}(\hat\mu_{t}), \\
		&\EE{\hat\nu_{t}} = \text{Var}(\hat\mu_{t}^c),\\
		& \Prob{\hat\mu_{t}^c -  V_{v,t,q}^{(\alpha)}\sqrt{{\hat\nu_{t}}} \le \mu } \ge 1- 1/v^\alpha,  \text{ and} \\
		& \Prob{\hat\mu_{t}^c +V_{v,t,q}^{(\alpha)}\sqrt{{\hat\nu_{t}}} \ge \mu} \ge 1- 1/v^\alpha,
	}
	where ${\rho^2} = \sigma_{X\boldsymbol{W}}\Sigma_{\boldsymbol{W}\boldsymbol{W}}^{-1}\sigma_{\boldsymbol{W}X}^\top/\sigma^2$ is the square of the multiple correlation coefficient, $\sigma^2 = \text{Var}(X)$, $\sigma_{X\boldsymbol{W}} = (\text{Cov}(X,W_{1}), \ldots, \text{Cov}(X,W_{q}))$, and $V_{v,t,q}^{(\alpha)}$ is the $100(1-1/v^\alpha)^{\text{th}}$ percentile value of the $t-$distribution with $t-q-1$ degrees of freedom.
	
\end{fact}

\section{Missing proofs involving Estimator with One Control Variate}
\label{asec:oneCV}

\varEst*
\begin{proof}
	For simplicity of notation, we are dropping arm index $i$ from the subscript of $W$ in this proof. Using \cref{fact:estmProp}, we know the variance of estimator for arm $i$ with one control variate is
	\eqs{
		\text{Var}(\hat\mu_{s,i}^c) = \sigma_{c,i}^2 (\boldsymbol{Y}^\top\boldsymbol{Y})_{11}^{-1},
	}
	where Var$(\bar{X}_i )= \sigma_{c,i}^2$ and
	$$
	\boldsymbol{Y} = \begin{pmatrix}
		1~~~ &\omega - W_{1}\\
		\vdots &\vdots\\
		1~~~ &\omega - W_{s}
	\end{pmatrix}.
	$$
	Now we compute the value of $(\boldsymbol{Y}^\top\boldsymbol{Y})_{11}^{-1}$. First we get the $\boldsymbol{Y}^\top\boldsymbol{Y}$ as follows:
	\als{
		\boldsymbol{Y}^\top\boldsymbol{Y} = 
		\begin{pmatrix}
			1 ~~~~~~~\ldots ~~~~~~~~ 1\\
			\omega - W_{1} ~~\ldots~~~ \omega - W_{s}
		\end{pmatrix}
		\begin{pmatrix}
			1~~~ &\omega - W_{1}\\
			\vdots &\vdots\\
			1~~~ &\omega - W_{s}
		\end{pmatrix}
		= \begin{pmatrix}
			s &\sum_{r=1}^s (\omega - W_{r})\\\\
			\sum_{r=1}^s (\omega - W_{r}) &\sum_{r=1}^s (\omega - W_{r})^2
		\end{pmatrix}.
	}
	
	Now we the $(\boldsymbol{Y}^\top\boldsymbol{Y})^{-1}$ as follows:
	\als{
		(\boldsymbol{Y}^\top\boldsymbol{Y})^{-1} = \frac{1}{s\sum\limits_{r=1}^s (\omega - W_{r})^2 - \left(\sum\limits_{r=1}^s (\omega - W_{r})\right)^2}
		\begin{pmatrix}
			\sum_{r=1}^s (\omega - W_{r})^2&-\sum_{r=1}^s (\omega - W_{r})\\\\
			-\sum_{r=1}^s (\omega - W_{r}) &s
		\end{pmatrix}.
	}
	
	The value of  $(\boldsymbol{Y}^\top\boldsymbol{Y})_{11}^{-1}$ given as
	\als{
		(\boldsymbol{Y}^\top\boldsymbol{Y})_{11}^{-1} &= \frac{\sum_{r=1}^s (\omega - W_{r})^2}{s\sum_{r=1}^s (\omega - W_{r})^2 - \left(\sum_{r=1}^s (\omega - W_{r})\right)^2} \\
		&=\frac{1}{s}\frac{1}{1- \frac{\left(\sum_{r=1}^s (\omega - W_{r})\right)^2}{s\sum_{r=1}^s (\omega - W_{r})^2} } \\
		&=\frac{1}{s} \left(1- \frac{\left(\sum_{r=1}^s (\omega - W_{r})\right)^2}{s\sum_{r=1}^s (\omega - W_{r})^2} \right)^{-1} \\
		&=\frac{1}{s} \left(1- \frac{\left(\sum_{r=1}^s (W_{r} - \omega)\right)^2}{s\sum_{r=1}^s (W_{r}-\omega )^2} \right)^{-1} \\
		\implies (\boldsymbol{Y}^\top\boldsymbol{Y})_{11}^{-1}  &= \frac{Z_s}{s},
	}
	where $Z_s = \left(1- \frac{\left(\sum_{r=1}^s (W_{r} - \omega)\right)^2}{s\sum_{r=1}^s (W_{r}-\omega )^2} \right)^{-1}$. 
	After $s$ number of observations of rewards and associated control variates from arm $i$, the estimator of Var$(\hat\mu)$ is given by
	\eqs{
		\hat{\nu}_{s,i} =  \frac{Z_{s} \hat\sigma_{c,i}^2(s)}{s}
	} 
	where $\hat\sigma_{c,i}^2(s)= \frac{1}{s-2}\sum_{r=1}^{s} (\bar{X_r} - \hat\mu_{s,i}^c)^2$ is the unbiased estimator of $ \sigma_{c,i}^2$ (by \cref{fact:varEst}). Further, $\EE{\hat{\nu}_{s,i}}  = \text{Var}(\hat\mu_{s,i}^c)$ (by \cref{fact:normalEstProp}) which implies that $\hat\nu_{s,i}$  is an unbiased estimator of $\text{Var}(\hat\mu_{s,i}^c)$.
\end{proof}

Theorem 1 of \cite{OR90_nelson1990control} shows that the empirical mean with control variate is an unbiased estimator for a given number of samples (as stated in \cref{lem:estProp}). Our next result adapts the concentration bound given in Theorem 1 of \cite{OR90_nelson1990control} for a given number of samples.

\confProb*

\begin{proof}
	The proof follows from \cref{fact:normalEstProp} with $q=1$ control variates and replacing other parameters with arms specific parameters for $s$ observations of arm rewards and associated control variate, i.e., $\hat\mu_s^c$ by $\hat\mu_{s,i}^c$, $\mu$ by $\mu_i$, and $\hat\nu_s$ by $\hat\nu_{s,i}$. Note that we use $t-$distribution for confidence intervals, hence the value of $\Prob{|\hat\mu_{s,i}^c - \mu_i| \ge V_{t,s,1}^{(\alpha)}\sqrt{{\hat\nu_{s,i}}}}$ depends only on the value of $V_{t,s,1}^{(\alpha)}$. Now using following simple algebraic manipulations, we have
	\als{
		\Prob{|\hat\mu_{s,i}^c - \mu_i| \ge V_{t,s,1}^{(\alpha)}\sqrt{{\hat\nu_{s,i}}}} &= 1 - \Prob{|\hat\mu_{s,i}^c - \mu_i| \le V_{t,s,1}^{(\alpha)}\sqrt{{\hat\nu_{s,i}}}} \\
		&\le 1 - \left(1-\frac{2}{t^\alpha}\right) & \text{(using \cref{fact:normalEstProp})} \\
		\implies \Prob{|\hat\mu_{s,i}^c - \mu_i| \ge V_{t,s,1}^{(\alpha)}\sqrt{{\hat\nu_{s,i}}}} & \le {2}/{t^\alpha}. \qedhere
	}
\end{proof}

\section{Estimator with Multiple Control Variates}
\label{asec:multiCV}
In this section, we extend our setup to case where the multiple control variates can be associated with an arm's rewards. Let $q$ be the number of control variates. Then the sample for arm $i$ with associated control variates is given by
\eqs{
	\bar{X}_{t,i,q} = X_{t,i} + \sum_{j=1}^q\beta^*_{i,j}(\omega_i - W_{t,i,j}),
}
where $W_{t,i,j}$ is the $j^{\text{th}}$ control variate of arm $i$ that is observed in round $t$, $\omega_{i,j}=\EE{W_{t,i,j}}$, and $\beta^*_{i,j}=\text{Cov}(X_i, W_{i,j})/\text{Var}(W_{i,j})$. The unbiased mean reward estimator for arm $i$ with $s$ samples of rewards and associated $q$ number of control variates with known values of $\beta^*_{i,j}$ is given by
\eqs{
	\hat\mu_{s,i,q}^c =\frac{1}{s} \sum_{r=1}^s \bar{X}_{r,i}.
}

Let $\hat{\mu}_{s,i} = \frac{1}{s} \sum_{r=1}^s X_{r,i}$, $\boldsymbol{\hat\beta}^*_i = \left(\hat\beta_{i,1}, \ldots, \hat\beta_{i,q}\right)^\top$ is the estimate of $\boldsymbol{\beta}_i^*$, $\boldsymbol{\omega}_{i} = \left(\omega_{i,1}, \ldots, \omega_{i,q}\right)^\top$, and $\boldsymbol{\hat\omega}_{s,i}= \left(\hat\omega_{s,i,1}, \ldots, \hat\omega_{s,i,q}\right)^\top$, where $\hat{\omega}_{s,i,j} = \frac{1}{s} \sum_{r=1}^s W_{r,i,j}$.  Then $\hat\mu_{s,i,q}^c$ can be written as:
\eq{
	\label{equ:estMeanReward}
	\hat\mu_{s,i,q}^c = \hat\mu_{s,i} + \boldsymbol{\hat\beta}_{i}^{*\top} (\boldsymbol{\omega}_{i} - \boldsymbol{\hat\omega}_{s,i}),
}

Let $s$ be the number of rewards and associated control variates  samples for arm $i$,  $\boldsymbol{W}_i$ be the $s\times q$ matrix whose $r^{\text{th}}$ row is $\left(W_{r,i,1}, W_{r,i,2}, \ldots, W_{r,i,q} \right)$, $\boldsymbol{S}_{W_iW_i}=(s-1)^{-1}(\boldsymbol{W}_i^\top\boldsymbol{W}_i - s\boldsymbol{\hat\omega}_{s,i}\boldsymbol{\hat\omega}_{s,i}^\top )$, and $\boldsymbol{S}_{X_iW_i}$ $=(s-1)^{-1}(\boldsymbol{W}_i^\top \boldsymbol{X}_i - s{\boldsymbol{\hat\omega}_i}~\hat\mu_{s,i})$ where $\boldsymbol{X}_i=(X_{1,i}, \ldots, X_{s,i})^\top$. Then by extending the arguments used in \cref{equ:estBeta} to get estimated coefficient for a scalar $\hat\beta_i^*$ to a vector $\boldsymbol{\hat\beta}_i^*$, the estimated coefficient vector is given by 
\eq{
	\boldsymbol{\hat\beta}^*_i = \boldsymbol{S}_{W_iW_i}^{-1} \boldsymbol{S}_{X_iW_i}.
}

Our next results are generalization of \cref{lem:varEst} and \cref{lem:confProb} to MAB-CV problems with $q$ control variates.

\begin{restatable}{lem}{varEstMulti}
	\label{lem:varEstMulti}
	Let $s$ be the number of rewards and associated control variate samples. Then, the sample variance of mean reward estimator for arm $i$ is given by
	\eqs{
		\hat\nu_{s,i,q} = \frac{Z_{s,q} \hat\sigma_{c,i,q}^2(s)}{s},
		~~\text{where}~~
		Z_{s,q} = \left(1 + \frac{(\boldsymbol{\hat\omega}_{s,i} - \boldsymbol{\omega}_{i})^\top\boldsymbol{S}_{W_iW_i}^{-1} (\boldsymbol{\hat\omega}_{s,i} - \boldsymbol{\omega}_{i})}{1-1/s}\right)
	} 
	and $ \hat\sigma_{c,i}^2(s)= \frac{1}{s-q-1}\sum_{r=1}^{s} (\bar{X}_r - \hat\mu_{s,i}^c)^2$. Further, $\hat\nu_{s,i,q}$ is  unbiased estimator. 
\end{restatable}

\begin{proof}
	Using \cref{fact:estmProp}, we know the variance of estimator for arm $i$ with one control variate is
	\eqs{
		\text{Var}(\hat\mu_{s,i}^c) = \sigma_{c,i}^2 (\boldsymbol{Y}^\top\boldsymbol{Y})_{11}^{-1}.
	}
	
	After having $s$ observations of rewards and associated $q$ control variates from arm $i$ , we use \cref{equ:varMulti} to get an unbiased estimator for $\text{Var}(\hat\mu_{s,i}^c)$ as follows:
	\eq{
		\hat\nu_{s,i,q} =  \frac{Z_{s,q} \hat\sigma_{c,i,q}^2(s)}{s}
	} 
	where $Z_{s,q} = \left(1 + \frac{(\boldsymbol{\hat\omega}_{s,i} - \boldsymbol{\omega}_i)^\top\boldsymbol{S}_{W_iW_i}^{-1} (\boldsymbol{\hat\omega}_{s,i} - \boldsymbol{\omega}_i)}{1-1/s}\right)$ and $ \hat\sigma_{c,i,q}^2(s)= \frac{1}{s-q-1}\sum_{r=1}^{s} (\bar{X}_r - \hat\mu_{s,i}^c)^2$ \citep{OR90_nelson1990control}. Also note that $\hat\sigma_{c,i,q}^2(t)$ is an unbiased estimator of Var$(\bar{X}_i)$ (by \cref{fact:varEst}).
\end{proof}

\begin{restatable}{lem}{confProbMulti}
	\label{lem:confProbMulti}
	Let $s$ be the number of rewards and associated control variates samples from arm $i$ in round $t$. Then
	\eqs{
		\Prob{|\hat\mu_{s,i,q}^c - \mu_i| \ge V_{t,s,q}^{(\alpha)}\sqrt{{\hat\nu_{s,i,q}}}} \le 2/t^\alpha.
	}
	where $V_{t,s,q}^{(\alpha)}$ is the $100(1-1/t^\alpha)^{\text{th}}$ percentile value of the $t-$distribution with $s-q-1$ degrees of freedom and $\hat\nu_{s,i,q}$ is an unbiased estimator for variance of $\hat\mu_{s,i,q}^c$.
\end{restatable}

\begin{proof}
	The proof follows from \cref{fact:normalEstProp} with $q$ control variates and replacing other parameters with arms specific parameters for $s$ observations of arm rewards and associated control variate, i.e., $\hat\mu_s^c$ by $\hat\mu_{s,i}^c$, $\mu$ by $\mu_i$, and $\hat\nu_s$ by $\hat\nu_{s,i}$. As we use $t-$distribution for confidence intervals, the value of $\Prob{|\hat\mu_{t,i,q}^c - \mu_i| \ge V_{t,s,q}^{(\alpha)}\sqrt{{\hat\nu_{s,q}}}}$ depends only on the value of $V_{t,s,q}^{(\alpha)}$. Now using following simple algebraic manipulations, we have
	\als{
		\Prob{|\hat\mu_{s,i,q}^c - \mu_i| \ge V_{t,s,q}^{(\alpha)}\sqrt{{\hat\nu_{s,i,q}}}} &= 1 - \Prob{|\hat\mu_{s,i,q}^c - \mu_i| \le V_{t,s,q}^{(\alpha)}\sqrt{{\hat\nu_{s,i,q}}}} \\
		&\le1 - \left(1-\frac{2}{t^\alpha}\right) &\hspace{-5mm} \text{(using \cref{fact:normalEstProp})} \\
		\implies \Prob{|\hat\mu_{s,i,q}^c - \mu_i| \ge V_{t,s,q}^{(\alpha)}\sqrt{{\hat\nu_{s,i,q}}}} &\le{2}/{t^\alpha}. \hfill ~ \qedhere
	}
\end{proof}

Next, we define optimistic upper bound for estimate of mean reward with $q$ control variates as follows: 
\begin{equation}
	\label{equ:UCB_Multi}
	\text{UCB}_{t,i,q} = \hat\mu_{N_i(t),i,q}^c + V_{t,N_i(t),q}^{(\alpha)}\sqrt{{\hat\nu_{N_i(t),i,q}}}.
\end{equation}

For multiple control variate case, we can use \ref{alg:UCB-CV} with $Q=q+2$ and replacing UCB$_{t,i}$ by UCB$_{t,i,q}$ as defined in \cref{equ:UCB_Multi}.

\section{Regret Analysis of \ref{alg:UCB-CV}}
\label{asec:analysis}

Similar to \cite{NOW12_bubeck2012regret}, our following result defines three events and shows that the sub-optimal arm is only selected if one of these events are true.
\begin{lem}
	\label{lem:armSelection}
	Let $N_i(t)$ be the number of times sub-optimal arm selected until $t$ rounds. A sub-optimal arm $i$ is selected by \ref{alg:UCB-CV} in round $t$ if at least one of the three following events must be true:	
	\als{
		&1.~ \hat\mu_{N_i^\star(t),i^\star,q}^c + V_{t,N_i^\star(t),q}^{(\alpha)}\sqrt{{\hat\nu_{N_i^\star(t),i^\star,q}}}  \le \mu_{i^\star},\\
		&2.~ \hat\mu_{N_i(t),i,q}^c - V_{t,N_i(t),q}^{(\alpha)}\sqrt{{\hat\nu_{N_i(t),i,q}}} > \mu_i, \text{ and}\\
		&3.~ N_i(t) < \frac{4(V_{T,N_i(t),q}^{(\alpha)})^2 Z_{N_i(t),q} \hat\sigma_{c,i,q}^2(N_i(t))}{\Delta_i^2}.
	}
\end{lem}

\begin{proof}
	Let assume that all three events are all false, then we have:
	\als{
		\text{UCB}_{t, i^\star} & = \hat\mu_{N_i^\star(t),i^\star,q}^c + V_{t,N_i^\star(t),q}^{(\alpha)}\sqrt{{\hat\nu_{N_i^\star(t),i^\star,q}}}  \\
		& > \mu_{i^\star} & (\text{if event 1 is false})\\
		& = \mu_{i} + \Delta_i \\
		& \ge \mu_{i} + 2V_{T,N_i(t),q}^{(\alpha)}\sqrt{\frac{Z_{N_i(t),q} \hat\sigma_{c,i,q}^2(N_i(t))}{N_i(t)}}. & (\text{if event 3 is false}) \\
		\intertext{From \cref{lem:varEstMulti}, we have $\hat\nu_{N_i(t),i,q} = {Z_{N_i(t),q} \hat\sigma_{c,i,q}^2(N_i(t))}/{N_i(t)}$}
		\text{UCB}_{t, i^\star} &> \mu_{i} + 2V_{T,N_i(t),q}^{(\alpha)}\sqrt{{\hat\nu_{N_i(t),i,q}}} \\
		& \ge \mu_{i} + 2V_{t,N_i(t),q}^{(\alpha)}\sqrt{{\hat\nu_{N_i(t),i,q}}} &\hspace{-4cm} \left(\text{$V_{t,N_i(t),q}^{(\alpha)}$ is increasing function of $t$ and $t\le T$}\right) \\
		& \ge \hat\mu_{N_i(t),i,q}^c - V_{t,N_i(t),q}^{(\alpha)}\sqrt{{\hat\nu_{N_i(t),i,q}}} + 2V_{t,N_i(t),q}^{(\alpha)}\sqrt{{\hat\nu_{N_i(t),i,q}}} & (\text{if event 2 is false}) \\
		& = \hat\mu_{N_i(t),i,q}^c + V_{t,N_i(t),q}^{(\alpha)}\sqrt{{\hat\nu_{N_i(t),i,q}}} \\
		& = \text{UCB}_{t,i} \\
		\implies \text{UCB}_{t, i^\star} &> \text{UCB}_{t,i}.
	}
	If all three events are false, then the optimal arm's UCB value is larger than the UCB value of the sub-optimal arm $i$. Therefore, at least one of three events need to be true if the sub-optimal arm $i$ is selected in round $t$. It concludes our proof.
\end{proof}

When the number of samples is random, randomness has to be taken into account. The following result gives the concentration bound on empirical mean with the random number of observations.

\begin{lem}
	\label{lem:randomBound}
	Let $t \in \N$, $\mu$ be the true mean, $\hat{\mu}_s$ be the empirical mean with $s (\le t)$ observations, and $f(t,s)=V_{t,s,q}^{(2)}\sqrt{{\hat\nu_{s,i,q}}}$. If $\Prob{|\hat{\mu}_s - \mu| \le f(t,s)} \le 1/t^2$ and $N_t$ is the random number such that $N_t \in [t]$, then
	$$\Prob{|\hat{\mu}_{N_t} - \mu| \le f(t,N_t)} \le \frac{1}{t^2}.$$
\end{lem}

\begin{proof}
	Using the conditioning argument, we have
	\als{
		\Prob{|\hat{\mu}_{N_t} - \mu| \le f(t,N_t)} &= \sum_{s=1}^t \Prob{|\hat{\mu}_{s} - \mu| \le f(t,s)| N_t = s} \Prob{N_t = s} \\
		 &\le \sum_{s=1}^t \frac{1}{t^2} ~\Prob{N_t = s} &\hspace{-3cm} \left(\text{as $\Prob{|\hat{\mu}_s - \mu| \le f(t,s)} \le 1/t^2$} \right)
		 \\
		 &=\frac{1}{t^2} \sum_{s=1}^t  \Prob{N_t = s}. 
	}
	As $N_t \in [t]$, $\sum_{s=1}^t  \Prob{N_t = s}   = 1$, we have $\Prob{|\hat{\mu}_{N_t} - \mu| \le f(t,N_t)} \le \frac{1}{t^2}$.
\end{proof}

\estProp*
\begin{proof}
	From \cref{fact:normalEstProp}, we know that the estimator  $(\hat\mu_{s,i,q}^c)$ and its variance estimator $(\hat\nu_{s,i,q})$ are unbiased estimator. Further, using \cref{fact:normalEstProp}, we can get the expression mentioned in \cref{lem:estProp} for the variance of  reward estimator $(\hat\nu_{s,i,q})$.
\end{proof}
 
Using \cref{lem:estProp}, now we can upper bound the number of times a sub-optimal arm $i$ is selected by \ref{alg:UCB-CV} in the following result.
\begin{restatable}{lem}{expectedPulls}
	\label{lem:expectedPulls}
	Let $\alpha=2$, $q$ be the number of control variates, and $N_i(T)$ be the number of times sub-optimal arm $i$ selected in $T$ rounds. If $C_{T,i,q} = \EE{\frac{N_i(T) - 2}{N_i(T)-q-2} \left(\frac{V_{T,N_i(T),q}^{(2)}}{V_{T,T,q}^{(2)}}\right)^2}$ then the expected number of times a sub-optimal arm $i \in [K]$ selected by \ref{alg:UCB-CV} in $T$ rounds is upper bounded by
	\als{
		\EE{N_i(T)} \le \frac{4(V_{T, T,q}^{(2)})^2 C_{T,i,q}(1-\rho_i)\sigma_i^2}{\Delta_i^2} + \frac{\pi^2}{3} + 1.
	}
\end{restatable}

\begin{proof}
	Let $u=\ceil{\frac{4(V_{T,N_i(T),q}^{(2)})^2 Z_{N_i(T),q} \hat\sigma_{c,i,q}^2(N_i(T))}{\Delta_i^2}}$ then the Event 3 is false for $N_i(T)=u$. Further the value of $u$ is random as its depends on $V_{T,N_i(T),q}^{(2)}$, $Z_{N_i(T),q}$, and $\hat\sigma_{c,i,q}^2(N_i(T)$. Now using \cref{lem:armSelection}, we have
	\al{
		\EE{N_i(T)}	& = \EE{\sum_{t=1}^T \one{I_t = i}} \nonumber \\
		& \le \EE{\sum_{t=1}^T \one{\text{Event 3 is true}}} + \EE{\sum_{t=1}^T \one{I_t = i \text{ and Event 3 is false}}} \nonumber \\
		& \le \EE{u} + \EE{\sum_{t=1}^T \one{\text{Either Event 1 or Event 2 is true}}} \nonumber \\
		\implies \EE{N_i(T)} & \le \EE{u} + \EE{\sum_{t=1}^T \one{\text{Event 1 is true}}} +  \EE{\sum_{t=1}^T \one{\text{Event 2 is true}}}.  \label{equ:expectedPull} 
	}
	
	We will first bound $\EE{\sum_{t=1}^T \one{\text{Event 1 is true}}}$ as follows
	\al{
		\EE{\sum_{t=1}^T \one{\text{Event 1 is true}}} &= \sum_{t=1}^T \Prob{\text{Event 1 is true}} \nonumber \\
		&= \sum_{t=1}^T \Prob{\hat\mu_{N_i^\star(t),i^\star,q}^c + V_{t,N_i(t),q}^{(2)}\sqrt{{\hat\nu_{N_i^\star(t),i^\star,q}}}  \le \mu_{i^\star}} \nonumber \\
		&= \sum_{t=1}^T \Prob{\hat\mu_{N_i^\star(t),i^\star,q}^c - \mu_{i^\star} \le - V_{t,N_i(t),q}^{(2)}\sqrt{{\hat\nu_{N_i^\star(t),i^\star,q}}}} \nonumber \\
		&= \sum_{t=1}^T \frac{1}{t^2} &\hspace{-3cm} (\text{using \cref{lem:confProbMulti} and \cref{lem:randomBound}}) \nonumber \\
		& \le \sum_{t=1}^\infty \frac{1}{t^2} \nonumber \\
		\implies \EE{\sum_{t=1}^T \one{\text{Event 1 is true}}} & \le \frac{\pi^2}{6}. \label{equ:expE1bound}
	}
	
	Now we will bound $\EE{\sum_{t=1}^T \one{\text{Event 2 is true}}}$ as follows
	\al{
		\EE{\sum_{t=1}^T \one{\text{Event 2 is true}}} &= \sum_{t=1}^T \Prob{\text{Event 2 is true}} \nonumber \\
		&= \sum_{t=1}^T \Prob{\hat\mu_{N_i(t),i,q}^c - V_{t,N_i(t),q}^{(2)}\sqrt{{\hat\nu_{N_i(t),i,q}}} > \mu_i} \nonumber \\
		&= \sum_{t=1}^T \Prob{\hat\mu_{N_i(t),i,q}^c - \mu_i > V_{t,N_i(t),q}^{(2)}\sqrt{{\hat\nu_{N_i(t),i,q}}}} \nonumber \\
		&= \sum_{t=1}^T \frac{1}{t^2} &\hspace{-3cm} (\text{using \cref{lem:confProbMulti} and \cref{lem:randomBound}}) \nonumber \\
		& \le \sum_{t=1}^\infty \frac{1}{t^2} \nonumber \\
		\implies \EE{\sum_{t=1}^T \one{\text{Event 1 is true}}} & \le \frac{\pi^2}{6}. \label{equ:expE2bound}
	}
	
	Last, we will bound $\EE{u}$ as follows
	\als{
		\EE{u} &= \EE{\ceil{\frac{4(V_{T,N_i(T),q}^{(2)})^2 Z_{N_i(T),q} \hat\sigma_{c,i,q}^2(N_i(T))}{\Delta_i^2}}}\\
		&\le \EE{\frac{4(V_{T,N_i(T),q}^{(2)})^2 Z_{N_i(T),q} \hat\sigma_{c,i,q}^2(N_i(T))}{\Delta_i^2}} + 1 \\
		& = \frac{4}{\Delta_i^2}\EE{(V_{T,N_i(T),q}^{(2)})^2Z_{N_i(T),q} \hat\sigma_{c,i,q}^2(N_i(T))} + 1.
	}
	Divide and multiply LHS by $(V_{T,T,q}^{(2)})^2$ and $N_i(T)$, we have
	\als{
		\EE{u} & \le \frac{4(V_{T,T,q}^{(2)})^2}{\Delta_i^2}\EE{\left(\frac{V_{T,N_i(T),q}^{(2)}}{V_{T,T,q}^{(2)}}\right)^2 \frac{Z_{N_i(T),q} \hat\sigma_{c,i,q}^2(N_i(T))}{N_i(T)} N_i(T)} + 1.
	}
	Using $\hat\nu_{N_i(T),i,q} = \frac{Z_{N_i(T),q} \hat\sigma_{c,i,q}^2(N_i(T))}{N_i(T)}$ and the Law of Iterated Expectation, we get
	\als{
		\EE{u} & \le \frac{4(V_{T,T,q}^{(2)})^2}{\Delta_i^2}\EE{\left(\frac{V_{T,N_i(T),q}^{(2)}}{V_{T,T,q}^{(2)}}\right)^2\hat\nu_{N_i(T),i,q}  N_i(T)} + 1 &\hspace{-9mm}  \nonumber \\
		& = \frac{4(V_{T,T,q}^{(2)})^2}{\Delta_i^2}\EE{\EE{\left(\frac{V_{T,N_i(T),q}^{(2)}}{V_{T,T,q}^{(2)}}\right)^2\hat\nu_{N_i(T),i,q}  N_i(T)| N_i(T)}} + 1 \\
		\implies \EE{u} & = \frac{4(V_{T,T,q}^{(2)})^2}{\Delta_i^2}\EE{\left(\frac{V_{T,N_i(T),q}^{(2)}}{V_{T,T,q}^{(2)}}\right)^2N_i(T)\EE{\hat\nu_{N_i(T),i,q} | N_i(T)}} + 1.
	}
	
	We know that  $\hat\nu_{N_i(T),i,q}$ given $N_i(T)$ is an unbiased estimator (\cref{lem:varEstMulti}) of variance of reward estimator that used $N_i(T)$ observations. Further, the reward observations are IID hence Var$(\hat\mu_{s,i}) = \sigma_i^2/s$ where $\sigma_i^2= \text{Var}(X_{s,i})$ . With this, using \cref{lem:estProp} for getting $\EE{\hat\nu_{N_i(T),i,q}| N_i(T)}$, we have 
	\als{
		\EE{u} &  = \frac{4(V_{T,T,q}^{(2)})^2}{\Delta_i^2}\EE{\left(\frac{V_{T,N_i(T),q}^{(2)}}{V_{T,T,q}^{(2)}}\right)^2N_i(T)\frac{N_i(T) - 2}{N_i(T)-q-2} (1-\rho_i^2)\frac{\sigma^2}{N_i(T)}} + 1  \\
		& = \frac{4(V_{T,T,q}^{(2)})^2 (1-\rho_i^2)\sigma_i^2}{\Delta_i^2} \EE{\frac{N_i(T) - 2}{N_i(T)-q-2} \left(\frac{V_{T,N_i(T),q}^{(2)}}{V_{T,T,q}^{(2)}}\right)^2} + 1.
	}
	Since $C_{T,i,q} =  \EE{\frac{N_i(T) - 2}{N_i(T)-q-2} \left(\frac{V_{T,N_i(T),q}^{(2)}}{V_{T,T,q}^{(2)}}\right)^2}$, we have
	\al{
		\implies \EE{u} & \le \frac{4(V_{T,T,q}^{(2)})^2 C_{T,i,q}(1-\rho_i^2)\sigma_i^2}{\Delta_i^2} + 1. \label{equ:expU1bound}
	}
	Using \cref{equ:expE1bound},  \cref{equ:expE2bound}, and \cref{equ:expU1bound},  in \cref{equ:expectedPull}, we get
	\eqs{
		\EE{N_i(T)} \le \frac{4(V_{T,T,q}^{(2)})^2 C_{T,i,q}(1-\rho_i^2)\sigma_i^2}{\Delta_i^2} + \frac{\pi^2}{3} + 1. \qedhere
	}
\end{proof}

Now we are ready to upper bound the regret of \ref{alg:UCB-CV}.

\regretBound*
\begin{proof}
	By definition, we have
	\als{
		{\Regret_T} &= T\mu_{i^\star} - \EE{\sum_{t=1}^{T} X_{t,I_t}} = \sum_{i=1}^K \Delta_i \EE{N_i(T)}. 
	}
	
	Using \cref{lem:expectedPulls} to replace $\EE{N_i(T)}$, we get
	\eqs{
		{\Regret_T} \le \sum_{i \in [K]\setminus \{i^\star\}} \left( \frac{4(V_{T,T,q}^{(2)})^2 C_{T,i,q}(1-\rho_i^2)\sigma_i^2}{\Delta_i} + \frac{\Delta_i\pi^2}{3} + \Delta_i\right). \qedhere
	}
\end{proof}

\section{Batching}
\label{asec:batching}
The batching is well known method used for calculating confidence intervals for means of a sequence of correlated observations \citep{MS63_conway1963some,OR82_schmeiser1982batch}. Batching transforms correlated observations into smaller uncorrelated and (almost) normally distributed observations called batch means. Let $\bar{X}_{1,i}, \bar{X}_{2,i}, \ldots, \bar{X}_{s,i}$ be the correlated observations from arm $i \in [K]$. These can be transformed into $b$ batch means, where each batch mean uses $m$ correlated observations. The value of $b=\floor{s/m}$. The $j^{\text{th}}$ mean batch is given by
\eqs{
	Y_{j,i,q}^{\text{B}} = \frac{1}{m}\sum_{i=(j-1)m + 1}^{jm} \bar{X}_{s,i,q}, ~~~ j \in \{1, 2, \ldots, b\}.
}

Let the mean reward estimator be $\hat\mu_{s,i,q}^{c, \text{B}} = \frac{1}{\floor{s/m}} \sum_{j=1}^{\floor{s/m}}Y_{j,i,q}^{\text{B}}$ and $\hat\nu_{N_i(t),i,q}^{\text{B}}$ be the sample variance of mean reward estimator for arm $i$. Under assumption that batch means are normally distributed, we can define the optimistic upper bound for mean reward estimate as follows: 
\eqs{
	\text{UCB}_{t,i,q}^{\text{B}} = \hat\mu_{M_i(t),i,q}^{c, \text{B}} + V_{t,M_i(t),q}^{\text{B},\alpha}\sqrt{\hat\nu_{N_i(t),i,q}^{\text{B}}},
}
where $M_i(t)$ is the number of batch means until round $t$, $V_{t,s,q}^{\text{B},\alpha}$ is the $100(1-1/t^\alpha)^{\text{th}}$ percentile value of the $t-$distribution with $\floor{s/m} - q-1$ degrees of freedom, and $\hat\nu_{s,i,q}^{\text{G}}$ is computed using \cref{lem:varEstMulti} with $\floor{s/m}$ observations. We can use \ref{alg:UCB-CV} with $Q=mq-q +1$ and replacing $\text{UCB}_{t,i}$ by $\text{UCB}_{t,i,q}^{\text{B}}$ for solving the MAB-CV problem instance using batching method. 

The cost for batching is the loss of degrees of freedom. \citet{EJOR89_nelson1989batch} quantifies this cost in terms of the variance estimator when $\{\bar{X}_i\}$ is supposed to be having a multivariate normal distribution. Unfortunately, this assumption only holds for a minimum batch size $(m^\star)$ that makes batch means normally distributed uncorrelated observations. The value of $m^\star$ depends on the observations and there exists no closed-form expression for it.

We refer the readers to \cite{EJOR89_nelson1989batch} to know how to select the batch size in batching method. More detailed discussion about batching, jackknifing, and splitting can be found in \cite{OR90_nelson1990control}.

~\\
\hrule

%% file: main.bbl
\begin{thebibliography}{42}
\providecommand{\natexlab}[1]{#1}
\providecommand{\url}[1]{\texttt{#1}}
\expandafter\ifx\csname urlstyle\endcsname\relax
  \providecommand{\doi}[1]{doi: #1}\else
  \providecommand{\doi}{doi: \begingroup \urlstyle{rm}\Url}\fi

\bibitem[Abbasi-Yadkori et~al.(2011)Abbasi-Yadkori, P{\'a}l, and
  Szepesv{\'a}ri]{NIPS11_abbasi2011improved}
Yasin Abbasi-Yadkori, D{\'a}vid P{\'a}l, and Csaba Szepesv{\'a}ri.
\newblock Improved algorithms for linear stochastic bandits.
\newblock In \emph{Advances in Neural Information Processing Systems}, pages
  2312--2320, 2011.

\bibitem[Agrawal and Goyal(2012)]{COLT12_agrawal2012analysis}
Shipra Agrawal and Navin Goyal.
\newblock Analysis of thompson sampling for the multi-armed bandit problem.
\newblock In \emph{Conference on Learning Theory}, pages 39--1, 2012.

\bibitem[Agrawal and Goyal(2013)]{AISTATS13_agrawal2013further}
Shipra Agrawal and Navin Goyal.
\newblock Further optimal regret bounds for thompson sampling.
\newblock In \emph{Artificial intelligence and statistics}, pages 99--107,
  2013.

\bibitem[Audibert et~al.(2009{\natexlab{a}})Audibert, Munos, and
  Szepesv\'ari]{TCS09_UCBV_Audibert}
Jean-Yves Audibert, R\'emi Munos, and Csaba Szepesv\'ari.
\newblock Exploration–exploitation tradeoff using variance estimates in
  multi-armed bandits.
\newblock \emph{Theoretical Computer Science}, 410\penalty0 (19):\penalty0 1876
  -- 1902, 2009{\natexlab{a}}.

\bibitem[Audibert et~al.(2009{\natexlab{b}})Audibert, Munos, and
  Szepesv{\'a}ri]{TCS09_audibert2009exploration}
Jean-Yves Audibert, R{\'e}mi Munos, and Csaba Szepesv{\'a}ri.
\newblock Exploration--exploitation tradeoff using variance estimates in
  multi-armed bandits.
\newblock \emph{Theoretical Computer Science}, 410\penalty0 (19):\penalty0
  1876--1902, 2009{\natexlab{b}}.

\bibitem[Auer and Ortner(2010)]{PMH10_auer2010ucb}
Peter Auer and Ronald Ortner.
\newblock Ucb revisited: Improved regret bounds for the stochastic multi-armed
  bandit problem.
\newblock \emph{Periodica Mathematica Hungarica}, 61\penalty0 (1-2):\penalty0
  55--65, 2010.

\bibitem[Auer et~al.(2002)Auer, Cesa-Bianchi, and Fischer]{ML02_auer2002finite}
Peter Auer, Nicolo Cesa-Bianchi, and Paul Fischer.
\newblock Finite-time analysis of the multiarmed bandit problem.
\newblock \emph{Machine learning}, pages 235--256, 2002.

\bibitem[Avramidis et~al.(1991)Avramidis, Bauer~Jr, and
  Wilson]{NRL91_avramidis1991simulation}
Athanassios~N Avramidis, Kenneth~W Bauer~Jr, and James~R Wilson.
\newblock Simulation of stochastic activity networks using path control
  variates.
\newblock \emph{Naval Research Logistics (NRL)}, 38\penalty0 (2):\penalty0
  183--201, 1991.

\bibitem[Botev and Ridder(2014)]{Willy14_botev2014variance}
Zdravko Botev and Ad~Ridder.
\newblock Variance reduction.
\newblock \emph{Wiley StatsRef: Statistics Reference Online}, pages 1--6, 2014.

\bibitem[Bubeck et~al.(2012)Bubeck, Cesa-Bianchi,
  et~al.]{NOW12_bubeck2012regret}
S{\'e}bastien Bubeck, Nicolo Cesa-Bianchi, et~al.
\newblock Regret analysis of stochastic and nonstochastic multi-armed bandit
  problems.
\newblock \emph{Foundations and Trends{\textregistered} in Machine Learning},
  5\penalty0 (1):\penalty0 1--122, 2012.

\bibitem[Chapelle and Li(2011)]{NIPS11_chapelle2011empirical}
Olivier Chapelle and Lihong Li.
\newblock An empirical evaluation of thompson sampling.
\newblock In \emph{Advances in neural information processing systems}, pages
  2249--2257, 2011.

\bibitem[Chen and Ghahramani(2016)]{ICML16_chen2016scalable}
Yutian Chen and Zoubin Ghahramani.
\newblock Scalable discrete sampling as a multi-armed bandit problem.
\newblock In \emph{International Conference on Machine Learning}, pages
  2492--2501. PMLR, 2016.

\bibitem[Chu et~al.(2011)Chu, Li, Reyzin, and
  Schapire]{AISTATS11_chu2011contextual}
Wei Chu, Lihong Li, Lev Reyzin, and Robert~E Schapire.
\newblock Contextual bandits with linear payoff functions.
\newblock In \emph{International Conference on Artificial Intelligence and
  Statistics}, pages 208--214, 2011.

\bibitem[Conway(1963)]{MS63_conway1963some}
Richard~W Conway.
\newblock Some tactical problems in digital simulation.
\newblock \emph{Management science}, 10\penalty0 (1):\penalty0 47--61, 1963.

\bibitem[Dani et~al.(2008)Dani, Hayes, and Kakade]{COLT08_dani2008stochastic}
Varsha Dani, Thomas~P Hayes, and Sham~M Kakade.
\newblock Stochastic linear optimization under bandit feedback.
\newblock In \emph{COLT}, pages 355--366, 2008.

\bibitem[Efron(1982)]{Book82_efron1982jackknife}
Bradley Efron.
\newblock \emph{The jackknife, the bootstrap and other resampling plans}.
\newblock SIAM, 1982.

\bibitem[Efron and Gong(1983)]{TAS83_efron1983leisurely}
Bradley Efron and Gail Gong.
\newblock A leisurely look at the bootstrap, the jackknife, and
  cross-validation.
\newblock \emph{The American Statistician}, 37\penalty0 (1):\penalty0 36--48,
  1983.

\bibitem[Filippi et~al.(2010)Filippi, Cappe, Garivier, and
  Szepesv{\'a}ri]{NIPS10_filippi2010parametric}
Sarah Filippi, Olivier Cappe, Aur{\'e}lien Garivier, and Csaba Szepesv{\'a}ri.
\newblock Parametric bandits: The generalized linear case.
\newblock In \emph{Advances in Neural Information Processing Systems}, pages
  586--594, 2010.

\bibitem[Garivier and Capp{\'e}(2011)]{COLT11_garivier2011kl}
Aur{\'e}lien Garivier and Olivier Capp{\'e}.
\newblock The kl-ucb algorithm for bounded stochastic bandits and beyond.
\newblock In \emph{Proceedings of the 24th annual Conference On Learning
  Theory}, pages 359--376, 2011.

\bibitem[Hayashi(2000)]{Book_hayashi2000econometrics}
F~Hayashi.
\newblock \emph{Econometrics}.
\newblock Princeton University Press, 2000.

\bibitem[Honda and Takemura(2010)]{COLT10_honda2010asymptotically}
Junya Honda and Akimichi Takemura.
\newblock An asymptotically optimal bandit algorithm for bounded support
  models.
\newblock In \emph{COLT}, pages 67--79. Citeseer, 2010.

\bibitem[James(1985)]{JORS85_james1985variance}
BAP James.
\newblock Variance reduction techniques.
\newblock \emph{Journal of the Operational Research Society}, 36\penalty0
  (6):\penalty0 525--530, 1985.

\bibitem[Jun et~al.(2017)Jun, Bhargava, Nowak, and
  Willett]{NIPS17_jun2017scalable}
Kwang-Sung Jun, Aniruddha Bhargava, Robert Nowak, and Rebecca Willett.
\newblock Scalable generalized linear bandits: Online computation and hashing.
\newblock In \emph{Advances in Neural Information Processing Systems}, pages
  99--109, 2017.

\bibitem[Kaufmann et~al.(2012)Kaufmann, Korda, and
  Munos]{ALT12_kaufmann2012thompson}
Emilie Kaufmann, Nathaniel Korda, and R{\'e}mi Munos.
\newblock Thompson sampling: An asymptotically optimal finite-time analysis.
\newblock In \emph{International Conference on Algorithmic Learning Theory},
  pages 199--213. Springer, 2012.

\bibitem[Kreutzer et~al.(2017)Kreutzer, Sokolov, and
  Riezler]{ACL17_kreutzer2017bandit}
Julia Kreutzer, Artem Sokolov, and Stefan Riezler.
\newblock Bandit structured prediction for neural sequence-to-sequence
  learning.
\newblock In \emph{Proceedings of the 55th Annual Meeting of the Association
  for Computational Linguistics (Volume 1: Long Papers)}, pages 1503--1513,
  2017.

\bibitem[Lattimore and Szepesv\'ari(2020)]{lattimore_szepesvari_2020}
Tor Lattimore and Csaba Szepesv\'ari.
\newblock \emph{Bandit Algorithms}.
\newblock Cambridge University Press, 2020.
\newblock \doi{10.1017/9781108571401}.

\bibitem[Lavenberg and Welch(1981)]{MS82_lavenberg1981perspective}
Stephen~S Lavenberg and Peter~D Welch.
\newblock A perspective on the use of control variables to increase the
  efficiency of monte carlo simulations.
\newblock \emph{Management Science}, 27\penalty0 (3):\penalty0 322--335, 1981.

\bibitem[Lavenberg et~al.(1982)Lavenberg, Moeller, and
  Welch]{OR82_lavenberg1982statistical}
Stephen~S Lavenberg, Thomas~L Moeller, and Peter~D Welch.
\newblock Statistical results on control variables with application to queueing
  network simulation.
\newblock \emph{Operations Research}, 30\penalty0 (1):\penalty0 182--202, 1982.

\bibitem[Li et~al.(2010)Li, Chu, Langford, and
  Schapire]{WWW10_li2010contextual}
Lihong Li, Wei Chu, John Langford, and Robert~E Schapire.
\newblock A contextual-bandit approach to personalized news article
  recommendation.
\newblock In \emph{Proceedings of the 19th international conference on World
  wide web}, pages 661--670. ACM, 2010.

\bibitem[Li et~al.(2017)Li, Lu, and Zhou]{ICML17_li2017provably}
Lihong Li, Yu~Lu, and Dengyong Zhou.
\newblock Provably optimal algorithms for generalized linear contextual
  bandits.
\newblock In \emph{International Conference on Machine Learning}, pages
  2071--2080, 2017.

\bibitem[Mukherjee et~al.(2018)Mukherjee, Naveen, Sudarsanam, and
  Ravindran]{AAAI18_mukherjee2018efficient}
Subhojyoti Mukherjee, KP~Naveen, Nandan Sudarsanam, and Balaraman Ravindran.
\newblock Efficient-ucbv: An almost optimal algorithm using variance estimates.
\newblock In \emph{Proceedings of the AAAI Conference on Artificial
  Intelligence}, volume~32, 2018.

\bibitem[Nelson(1989)]{EJOR89_nelson1989batch}
Barry~L Nelson.
\newblock Batch size effects on the efficiency of control variates in
  simulation.
\newblock \emph{European Journal of Operational Research}, 43\penalty0
  (2):\penalty0 184--196, 1989.

\bibitem[Nelson(1990)]{OR90_nelson1990control}
Barry~L Nelson.
\newblock Control variate remedies.
\newblock \emph{Operations Research}, 38\penalty0 (6):\penalty0 974--992, 1990.

\bibitem[Pasupathy et~al.(2012)Pasupathy, Schmeiser, Taaffe, and
  Wang]{IIE12_pasupathy2012control}
Raghu Pasupathy, Bruce~W Schmeiser, Michael~R Taaffe, and Jin Wang.
\newblock Control-variate estimation using estimated control means.
\newblock \emph{IIE Transactions}, 44\penalty0 (5):\penalty0 381--385, 2012.

\bibitem[Perchet and Rigollet(2013)]{AS13_perchet2013multi}
Vianney Perchet and Philippe Rigollet.
\newblock The multi-armed bandit problem with covariates.
\newblock \emph{The Annals of Statistics}, 41\penalty0 (2):\penalty0 693--721,
  2013.

\bibitem[Rusmevichientong and
  Tsitsiklis(2010)]{MOR10_rusmevichientong2010linearly}
Paat Rusmevichientong and John~N Tsitsiklis.
\newblock Linearly parameterized bandits.
\newblock \emph{Mathematics of Operations Research}, 35\penalty0 (2):\penalty0
  395--411, 2010.

\bibitem[Schmeiser(1982)]{OR82_schmeiser1982batch}
Bruce Schmeiser.
\newblock Batch size effects in the analysis of simulation output.
\newblock \emph{Operations Research}, 30\penalty0 (3):\penalty0 556--568, 1982.

\bibitem[Schmeiser et~al.(2001)Schmeiser, Taaffe, and
  Wang]{IEE01_schmeiser2001biased}
Bruce~W Schmeiser, Michael~R Taaffe, and Jin Wang.
\newblock Biased control-variate estimation.
\newblock \emph{IIE Transactions}, 33\penalty0 (3):\penalty0 219--228, 2001.

\bibitem[Thompson(1933)]{BIOMETRIKA33_thompson1933likelihood}
William~R Thompson.
\newblock On the likelihood that one unknown probability exceeds another in
  view of the evidence of two samples.
\newblock \emph{Biometrika}, 25\penalty0 (3/4):\penalty0 285--294, 1933.

\bibitem[Van De~Geer(2005)]{ESBS05_van2005estimation}
Sara~A Van De~Geer.
\newblock Least squares estimation.
\newblock \emph{Encyclopedia of statistics in behavioral science}, 2005.

\bibitem[Vlassis et~al.(2019)Vlassis, Bibaut, Dimakopoulou, and
  Jebara]{ICML19_vlassis2019design}
Nikos Vlassis, Aurelien Bibaut, Maria Dimakopoulou, and Tony Jebara.
\newblock On the design of estimators for bandit off-policy evaluation.
\newblock In \emph{International Conference on Machine Learning}, pages
  6468--6476. PMLR, 2019.

\bibitem[Zhang et~al.(2016)Zhang, Yang, Jin, Xiao, and
  Zhou]{ICML16_zhang2016online}
Lijun Zhang, Tianbao Yang, Rong Jin, Yichi Xiao, and Zhi-hua Zhou.
\newblock Online stochastic linear optimization under one-bit feedback.
\newblock In \emph{International Conference on Machine Learning}, pages
  392--401, 2016.

\end{thebibliography}
